\pdfoutput=1

\documentclass[letterpaper]{article} 
\usepackage{aaai23}  
\usepackage{times}  
\usepackage{helvet}  
\usepackage{courier}  
\usepackage[hyphens]{url}  
\usepackage{graphicx} 
\usepackage{natbib}  
\usepackage{caption} 
\usepackage{algorithm}
\usepackage{algorithmic}

\usepackage{newfloat}
\usepackage{listings}
\DeclareCaptionStyle{ruled}{labelfont=normalfont,labelsep=colon,strut=off} 
\floatstyle{ruled}
\newfloat{listing}{tb}{lst}{}
\floatname{listing}{Listing}
\usepackage{microtype}
\usepackage{graphicx}
\usepackage{subfigure}
\usepackage{booktabs} 
\usepackage{amsfonts}
\usepackage{amsmath}
\usepackage{amsthm}
\usepackage{bbm}
\usepackage{bm}

\newtheorem{innercustomthm}{Theorem}
\newenvironment{customthm}[1]
  {\renewcommand\theinnercustomthm{#1}\innercustomthm}
  {\endinnercustomthm}

\newtheorem{theorem}{Theorem}
\newtheorem{lemma}{Lemma}

\newtheorem*{remark}{Remark}
\newtheorem{definition}{Definition}

\newcommand{\normmm}[1]{{\left\vert\kern-0.25ex\left\vert\kern-0.25ex\left\vert #1
   \right\vert\kern-0.25ex\right\vert\kern-0.25ex\right\vert}}
\newcommand{\normvec}[1]{\Vert #1 \Vert}

\title{WAT: Improve the Worst-class Robustness in Adversarial Training}
\author {
    Boqi Li\textsuperscript{\rm 1},
    Weiwei Liu\textsuperscript{\rm 1}\footnote{Corrsponding Author.} 
}
\affiliations {
    \textsuperscript{\rm 1}School of Computer Science, Wuhan University, China\\
    \{lbq988, liuweiwei863\}@gmail.com
}

\usepackage{bibentry}

\begin{document}

\maketitle

\begin{abstract}
Deep Neural Networks (DNN) have been shown to be vulnerable to adversarial examples. Adversarial training (AT) is a popular and effective strategy to defend against adversarial attacks. Recent works \cite{analysis:benz,frl,analysis:tian} have shown that a robust model well-trained by AT exhibits a remarkable robustness disparity among classes, and propose various methods to obtain consistent robust accuracy across classes. Unfortunately, these methods sacrifice a good deal of the average robust accuracy. Accordingly, this paper proposes a novel framework of worst-class adversarial training and leverages no-regret dynamics to solve this problem. Our goal is to obtain a classifier with great performance on worst-class and sacrifice just a little average robust accuracy at the same time. We then rigorously analyze the theoretical properties of our proposed algorithm, and the generalization error bound in terms of the worst-class robust risk. Furthermore, we propose a measurement to evaluate the proposed method in terms of both the average and worst-class accuracies. Experiments on various datasets and networks show that our proposed method outperforms the state-of-the-art approaches.
\end{abstract}

\section{Introduction}
Deep Neural Networks (DNNs) are known to be vulnerable to adversarial examples \cite{adv:Szegedy,fgsm}. An adversarial example in a small perturbation from test data can easily fool the DNN model, which remains a security issue and is unacceptable in some applications of DNN, such as road sign classification \cite{adv:physical} , text classification \cite{adv:text}, self-supervised learning \cite{certCL} and object detection \cite{adv:tshirt}. 

Numerous works \cite{cert,pgd,odeAdv} have attempted to improve the model robustness with various defenses. Adversarial Training (AT) \cite{fgsm,pgd} is one of the most widely used and effective methods of defense. AT generates adversarial examples from the training data in every mini-batch, then uses these examples to replace training data or adds them into the training data during the training phase.

Although AT obtains great average adversarial robustness performance over classes, \cite{analysis:benz,frl,analysis:tian} find that a robust model well-trained by AT exhibits a large robustness disparity in different classes on various balanced datasets, like the left classifier in Figure \ref{fig:overview}. Thus, AT leaves some classes vulnerable and may not perform well on some specific classes in certain real-world secure systems. For example, in the autonomous driving context, a classifier that has been well trained by AT may perform well on traffic sign classification and achieve great adversarial robustness performance on average while still exhibiting vulnerabilities on specific signs, which represents a potential danger for users.

Recently, some works \cite{analysis:benz,frl} have attempted to solve this problem. \citet{analysis:benz} analyze this phenomenon and use cost-sensitive learning to make the performance consistent over classes. \citet{frl} propose employing re-weight and re-margin strategies to solve this problem. Both of these methods obtain consistent robust accuracy over classes, but they sacrifice a good deal of the average robust accuracy, like middle classifier in Figure \ref{fig:overview}. To overcome the limitations of \citet{analysis:benz,frl}, this paper proposes a novel min-max learning paradigm to optimize worst-class robust risk and leverages no-regret dynamics to solve the proposed min-max problem, our goal is to achieve a classifier with great performance on worst-class but sacrifice a little average robust accuracy like the right classifier in Figure \ref{fig:overview}. Moreover, we rigorously analyze the theoretical properties of our proposed algorithm, and the generalization error bound in terms of the worst-class robust risk.
Empirically, we find that a trade-off exists between average and worst-class robust accuracies, and accordingly propose a measurement to evaluate the method in terms of both the average and worst-class accuracies.

\begin{figure*}[t]
\centering
\includegraphics[width=0.85\linewidth]{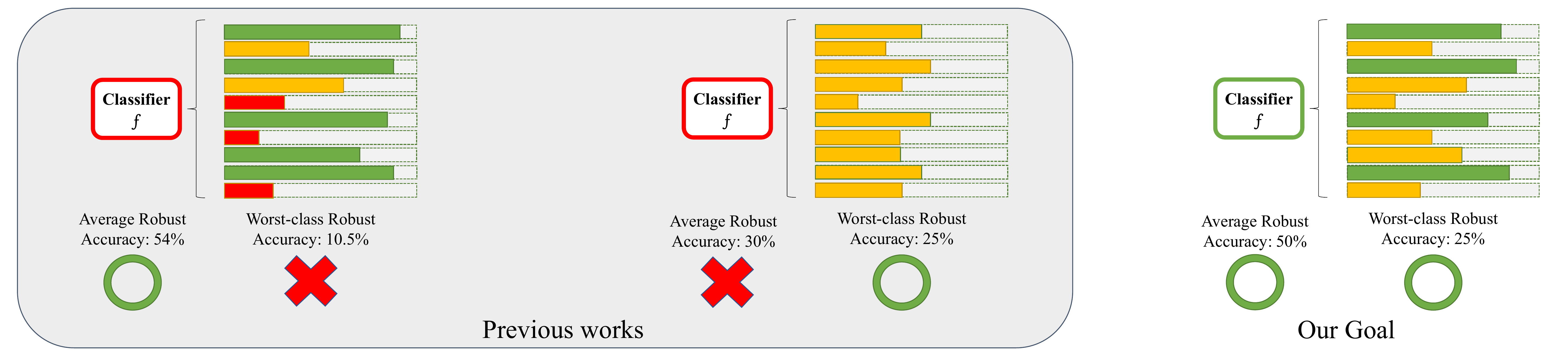}
\caption{A brief introduction of our main idea. Previous works only care about average or worst-class robust accuracy, while our method considers both worst-class and average robust accuracy.}
\label{fig:overview}
\end{figure*}

The main contributions in this paper are as follows:

\begin{itemize}

\item We propose a novel framework of worst-class adversarial training that leverages no-regret dynamics to solve the problem.

\item We analyze the theoretical properties of our proposed algorithm, and the generalization error bound in terms of the worst-class robust risk.

\item  A measurement is presented to evaluate the method in terms of both the average and worst-class accuracies.

\item Extensive experimental results on various datasets and networks verify that our proposed method outperforms state-of-the-art baselines.

\end{itemize}

\section{Related Work}
\label{relate}

$\textbf{Adversarial Robustness}.$ To improve adversarial robustness of DNN, adversarial training \cite{fgsm,pgd} is one of the most effective defenses. A large number of works \cite{trades,analysis:closelook,analysis:tradeoff} have explored the trade-off between robustness and accuracy. Amongst them, TRADES \cite{trades} is one of the most popular methods due to its promising experimental results. Besides, \citet{tradeoffFair} analyze the trade-off between robustness and fairness. \citet{theory:vcdim,theory:yin,theory:multi} theoretically analyze the adversarial robust generalization of a model while \citet{theory:firstorder} analyzes the first-order adversarial vulnerability of neural networks. Recently, a few works have been developed to further improve its performance, such as using unlabeled data \cite{unlabel-at}, feature alignments \cite{feature-at}, wider networks \cite{wrn-at} and a few tricks \cite{trick-at}.

$\textbf{Disparity of Class-wise Robustness}.$
In natural training, class-imbalance is a classical problem in long-tailed data. In such problem, major class has more data than minor class. Most of previous works to solve this problem can be concluded as resampling \cite{resample:zhou} and cost-sensitive learning \cite{csl:zou}. Recently, some works have opted to focus on the class-wise robustness disparity in the adversarial training. \citet{analysis:benz} study this problem empirically, and find that AT obtains a larger robust disparity among classes than that of natural training even in balanced data (e.g., CIFAR-10). \citet{analysis:tian} also find the similar experimental results on six different datasets. To solve this problem, \citet{analysis:benz} use a cost-sensitive learning fashion which is widely used in natural learning with imbalanced datasets; \citet{frl} propose a new method to reduce the class-wise variance of robust accuracy over classes. However their approaches both sacrifice a good deal of the average robust accuracy because they aim to make the performance consistent over classes. To address this issue, this paper aims to improve the worst-class adversarial robustness, while obtaining less average robust accuracy loss than previous works.

\section{Preliminaries}
\label{preliminaries}

This paper considers a $K$-class classification problem over input space $\mathcal{X}$ and output label space $\mathcal{Y} = \{1,2,\cdots,K \}$. Assume $\mathcal{D}$ is a distribution over $\mathcal{Z} = \mathcal{X} \times \mathcal{Y}$. We denote the sample as $ \mathcal{S}:\{\mathcal{X} \times \mathcal{Y}\}^n $. Let $\mathcal{F}$ be the hypothesis class, while $f(\mathbf{x};\theta) : \mathcal{X} \rightarrow \mathcal{Y}$ is a classifier in $\mathcal{F}$, where $\mathbf{x}$ is the input variable and $f$ is parametrized by $\theta$. Let $ \ell : \mathcal{F} \times \mathcal{Z} \rightarrow [0, B] $ be the loss function. Throughout this paper, we assume that $\ell$ is bounded. The expected natural risk $\mathcal{R}_{nat}(f)$ and expected robust risk $\mathcal{R}_{rob}(f)$ over distribution $\mathcal{D}$ and classifier $ f(\mathbf{x};\theta)$ can then be defined with respect to loss function $\ell$ as follows:
\begin{equation}
\mathcal{R}^{nat}(f) = \mathop{\mathbb{E}}\limits_{(\mathbf{x},y) \sim D} \ell(f(\mathbf{x} ;\theta), y) \label{nat_risk}
\end{equation}
\begin{equation}
\mathcal{R}^{rob}(f) = \mathop{\mathbb{E}}\limits_{(\mathbf{x},y) \sim D}  \max \limits_{\mathbf{x}^{\prime} \in \mathcal{B}(\mathbf{x},\epsilon)} \ell(f(\mathbf{x}^{\prime} ;\theta), y) \label{rob_risk}
\end{equation}
where $\mathcal{B}(\mathbf{x},\epsilon) = \{\mathbf{x}^{\prime}: ||\mathbf{x}^{\prime} - \mathbf{x} ||_p \leq \epsilon \}$ denotes the $\ell_p$-norm $(p \geq 1)$ ball centered at $\mathbf{x}$ with radius $\epsilon$.

\subsection{Worst-class Adversarial Robustness}

\begin{figure}[h]
\centering
\subfigure[PGD on CIFAR-10]{
    \begin{minipage}[t]{0.47\linewidth}
    \centering
    \includegraphics[width=1.0\linewidth]{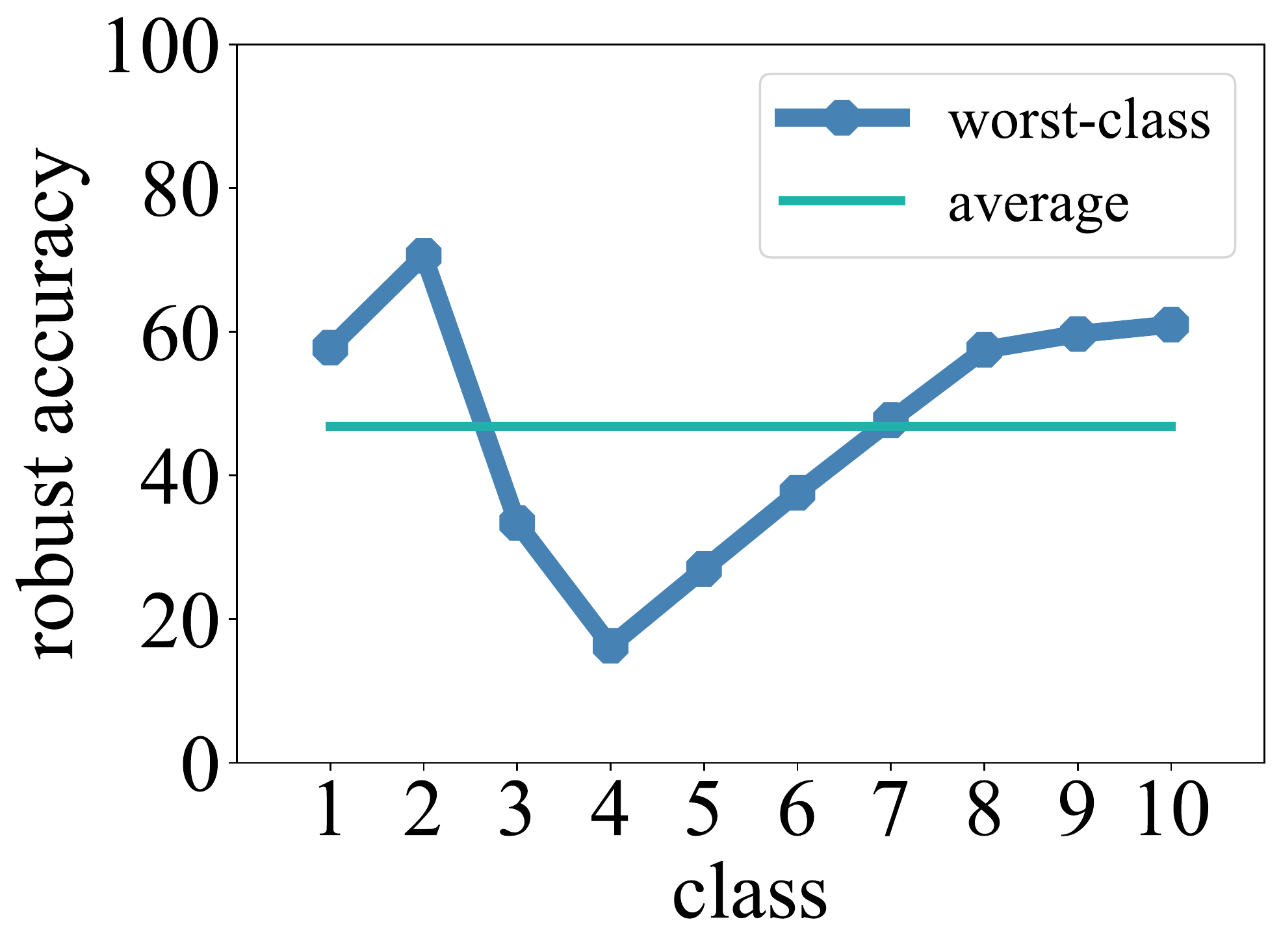}
    \label{fig:cifar-pgd}
    \end{minipage}
}
\subfigure[TRADES on CIFAR-10]{
    \begin{minipage}[t]{0.47\linewidth}
    \centering
    \includegraphics[width=1.0\linewidth]{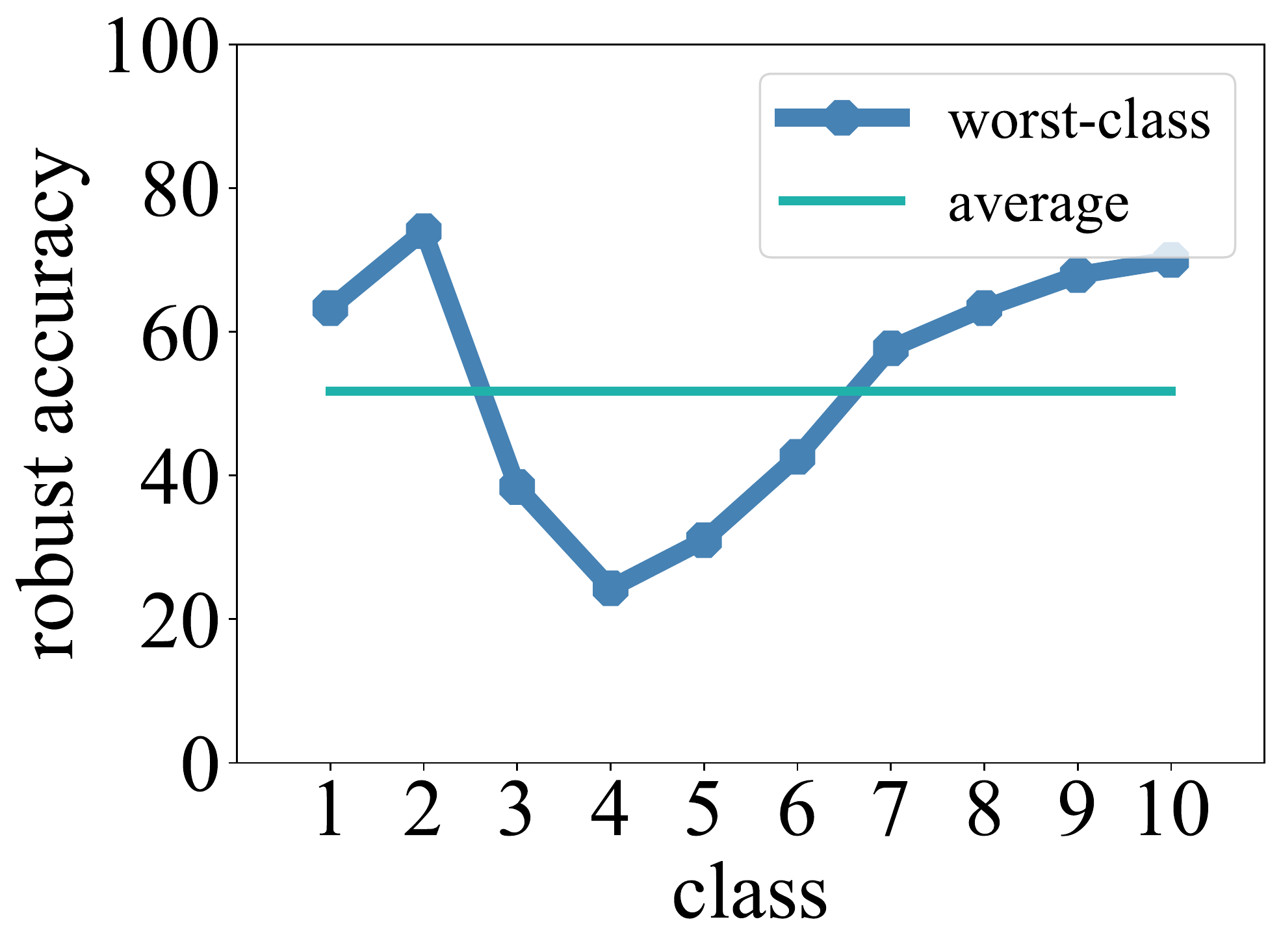}
    \label{fig:cifar-trades}
    \end{minipage}
}
\caption{ Class-wise robustness disparity of different AT using ResNet-18 on CIFAR-10. The robust accuracy (\%) is evaluated under PGD-20 attack.}
\label{fig:motivation}
\end{figure}

Typically, one aims to use ERM to obtain a good classifier from a hypothesis class with low empirical risk. However, a classifier with low empirical risk may not perform well on the worst class. To illustrate this phenomenon, we present the results of different AT variants on the CIFAR-10 in Figure \ref{fig:motivation}. From results in Figure \ref{fig:cifar-trades}, we can see that TRADES \cite{trades} obtains a worst-class robust accuracy of 23\% under PGD-20 \cite{pgd} attack, while the average robust accuracy of TRADES is 46\%. A similar phenomenon occurs when different variants of AT are used on different datasets. This degree of robustness disparity among classes is unacceptable in certain real-world secure systems. To study this problem, we define class-wise risk and worst-class risk as follows. We use $\mathcal{D}_k$ to denote the distribution of sample belonging to class $k$ class, and $\mathcal{S}_k$ to denote the sample drawn from $\mathcal{D}_k$.
\begin{equation}
\mathcal{R}_{k}^{nat}(f) = \mathop{\mathbb{E}}\limits_{(\mathbf{x},y) \sim \mathcal{D}_k } [ \ell(f(\mathbf{x} ;\theta), y)]  \label{class_nat_risk}
\end{equation}
\begin{equation}
\mathcal{R}_{k}^{rob}(f) = \mathop{\mathbb{E}}\limits_{(\mathbf{x},y) \sim \mathcal{D}_k }  [ \max \limits_{\mathbf{x}^{\prime} \in \mathcal{B}(\mathbf{x},\epsilon)} \ell(f(\mathbf{x}^{\prime} ;\theta), y)] \label{class_rob_risk}
\end{equation}

Similarly, we define the worst-class natural risk as $\mathcal{R}^{nat}_{wc}(f) = \max_{k \in [K]} \mathcal{R}^{nat}_k(f)$ and worst-class robust risk as $\mathcal{R}^{rob}_{wc}(f) = \max_{k \in [K]} \mathcal{R}^{rob}_k(f)$, where $[K]$ denotes the set of all positive integers in $[1, K]$. It follows that we have $
\mathcal{R}^{rob}_{wc}(f) \geq \mathcal{R}^{rob}(f) \geq \mathcal{R}^{nat}(f)
$.  

\subsection{Disparity of Adversarial Robustness}
\label{subsec:rho}

Figures \ref{fig:cifar-pgd} and \ref{fig:cifar-trades} show that a large gap exists between the worst-class robust accuracy and the average robust accuracy. Therefore, a classifier with low expected natural risk and expected robust risk may have high robust risk on some classes.

To solve this problem, recently, various strategies \cite{analysis:benz,frl} aimed at making the robust performance of the model consistent over all classes have been proposed. For example, \cite{frl} propose the re-weight and re-margin strategies on TRADES. Empirically, these works show that existing strategies typically sacrifice the average robust accuracy to improve worst-class robust accuracy. It is hard to choose proper weight for each class.

\begin{figure}[ht]
\centering
\includegraphics[width=0.3\textwidth]{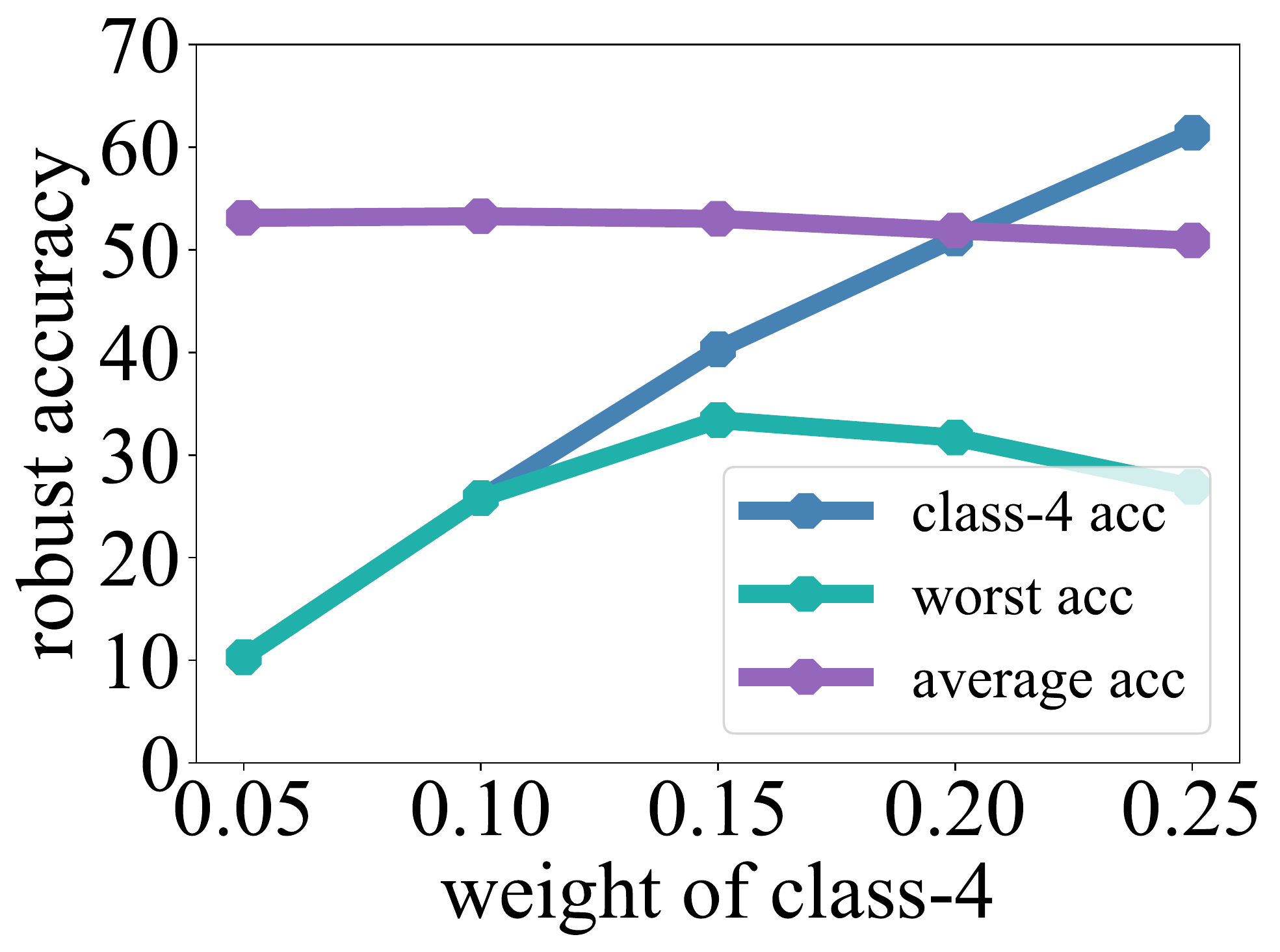}
  \caption{Trade-off between average and worst-class robust accuracy of ResNet-18 on CIFAR-10.}
  \label{fig:weight}
\end{figure}

In Figure \ref{fig:weight}, we use TRADES to train a ResNet-18 \cite{resnet} on CIFAR-10. We assign weight $w_k$ for class-$k$ and use a weighted loss $\sum_{k=1}^{K} w_k \ell_{trades}(\cdot, \cdot)$, where $\ell_{trades}(\cdot, \cdot)$ is the loss used in TRADES and is defined as $\ell_{trades} := \max_{\mathbf{x}^{\prime} \in B(\mathbf{x}, \epsilon)} CE(h_{\theta}(\mathbf{x}),y) + \beta KL(h_{\theta}(\mathbf{x}), h_{\theta}(\mathbf{x}^{\prime}))$. We change the weight of class-4 from 0.05 to 0.25 and set the weights of the other classes to be $(1-w_4)/(K-1)$. In Figure \ref{fig:cifar-pgd}, we find that the worst robust accuracy appears in class-4, so we choose to change the weights of class-4.

From the results in Figure \ref{fig:weight}, we can determine that when the weight of class-4 is increased from 0.05 to 0.15, the worst-class robust accuracy of TRADES grows by 23.1\%, while the average robust accuracy of TRADES drops by 0.09\%. Moreover, when the weight of class-4 is increased from 0.15 to 0.25, the worst-class and average robust accuracy drop at the same time. It is therefore demonstrably difficult to find the optimal weight for each class, and it is imperative to propose a measurement to simultaneously evaluate how much a given strategy would boost worst-class robust accuracy and decrease the average robust accuracy.

We use $\mathcal{A}$ to denote a vanilla adversarial training without any strategy, and $ \mathcal{A}_{\Delta}$ to denote adversarial training with the strategy $\Delta$. We run the algorithm $\mathcal{A}$ on hypothesis class $\mathcal{F}$ and sample $S_{train}$, and obtain the classifier $\hat f = \mathcal{A}(\mathcal{F}, \mathcal{S}_{train})$.

The average natural accuracy of a classifier $f$  with respect to distribution $\mathcal{D}$ is defined as
\begin{equation}
    Acc^{nat}(f,\mathcal{D}) = 1 - \mathbb{P}_{(\mathbf{x}, y) \sim \mathcal{D}}\left\{ y \neq f(\theta,\mathbf{x}) \right\}
\end{equation}
while average robust accuracy is defined as
\begin{equation}
    Acc^{rob}(f,\mathcal{D}) \! = \! 1  \!  -  \!   \mathbb{P}_{(\mathbf{x}, y) \! \sim \!  \mathcal{D}}\left\{\exists \mathbf{x}^{\prime} \! \in \! \mathcal{B}(\mathbf{x},\epsilon) \!,\! \text { s.t. }\!  y \!\neq\! f(\theta,\mathbf{x}^{\prime}) \right\}
\end{equation}
Similarly, we denote the $k$-th class natural accuracy as $Acc^{nat}_{k}(f,\mathcal{D})$, the worst-class natural accuracy as $Acc^{nat}_{wc}(f,\mathcal{D})$, the $k$-th class robust accuracy as $Acc^{rob}_{k}(f,\mathcal{D})$ and the worst-class robust accuracy as $Acc^{rob}_{wc}(f,\mathcal{D})$.
Let the average robust accuracy, the accuracy of the $k$-th class and the worst-class accuracy of a classifier $f$ on a test set $\mathcal{S}_{test}$ be $Acc^{rob}(f, \mathcal{S}_{test})$, $Acc_k(f, \mathcal{S}_{test}) $ and $Acc_{wc}(f, \mathcal{S}_{test}) $, respectively. For simplicity, we here use $\widehat{Acc}(f)$ to denote $Acc(f, \mathcal{S}_{test}) $.
This paper proposes a novel measurement to evaluate a method in terms of both the average and worst-class accuracy.
\begin{equation} \label{eq:rho}
\begin{aligned}
    \hat \rho( \mathcal{F}, \Delta, \mathcal{A}, \mathcal{S}) = &
    \frac{\widehat{Acc}_{wc}(\mathcal{A}_{\Delta}(\mathcal{F})) - \widehat{Acc}_{wc}(\mathcal{A}(\mathcal{F}))}
    {\widehat{Acc}_{wc}(\mathcal{A}(\mathcal{F}))} \\ -& 
    \frac{\widehat{Acc}(\mathcal{A}(\mathcal{F})) - \widehat{Acc}(\mathcal{A}_{\Delta}(\mathcal{F}))}
    {\widehat{Acc}(\mathcal{A}(\mathcal{F}))}
\end{aligned}
\end{equation}
Clearly, the larger the value of $\hat \rho$ is, the better a method performs.

\section{Proposed Method}
\label{method}

In this section, we formulate a novel min-max problem and then transform it into a two-player zero-sum game, and subsequently proposes a no-regret dynamics algorithm to solve the problem.

\subsection{No-regret Dynamics}
Consider a two-player zero-sum game, in which a decision-maker repeatedly plays a game against an adversary. More specifically, the decision-maker plays before the adversary and does not know the action taken by the adversary in each round. No-regret dynamics is one of the most efficient methods of achieving an $\epsilon$-coarse correlated equilibrium \cite{gametheory}.

Multiplicative Weight Updates Algorithm \cite{mwua} is one of the most widely used no-regret dynamic algorithms. Assume a game repeats for $T$ rounds, while the decision-maker has a choice of $n$ decisions. The decision-maker needs to repeatedly make a decision from the decision set and obtains an associated payoff from the adversary, while the best decision may not be known as a priori. Let $t = 1,2,\cdots, T$ denote the current round. In each round $t$, the decision-maker produces a distribution $\mathbf{p^{t}}$ over the decision set and chooses an action from the set according to $\mathbf{p^{t}}$. At this time, the adversary chooses a cost vector $\mathbf{C^{t}}$. Let $p^t_{k}$ be the $k$-th element of $\mathbf{p^{t}}$ while $C^t_{k}$ denotes the $k$-th element of $\mathbf{C^{t}}$.
Hedge Algorithm \cite{hedge} is one of Multiplicative Weights Updates Algorithm that uses an exponential function to adjust the weight of every decision as follows.
\begin{equation}
    p^{t}_{k} = \frac{\exp(\sum_{i=1}^{t-1} \eta C^{i}_{k} )}{\sum_{k=1}^{K} \exp(\sum_{i=1}^{t-1} \eta C^{i}_{k}) }.
\end{equation}
Clearly, Hedge Algorithm produces the weights depending on past performance.
Intuitively, this scheme works well because it tends to put heavy weights on high payoff decisions in the long run.

\subsection{Worst-class Adversarial Training}

The loss of a classifier $f$ on training set $\mathcal{S}_{tr}$ can be defined as
\begin{equation}
     L^{tr}_{0}(f) =  L^{tr}(f) = \frac{1}{\vert \mathcal{S}_{tr} \vert} \sum_{(\mathbf{x}_i,y_i) \in \mathcal{S}_{tr}} \ell_{trades}(f(\mathbf{x}_i ;\theta), y_i),
\end{equation}
where $\vert \cdot \vert$ denotes the cardinality of a set.
Let $L^{tr}_{k}(f)$ be the training loss on class $k$. Similarly, we use $L^{val}_{0}(f)$ and $L^{val}_{k}(f)$ to denote the loss of a classifier $f$ on the validation set $\mathcal{S}_{val}$ and validation loss on class $k$, respectively. $\ell_{trades}$ is the loss used in TRADES. 

We aim to minimize the following risk
\begin{equation}
 \min \limits_{f} \max_{k \in [0, K]} \mathcal{R}^{rob}_k(f),
\label{minmax}
\end{equation}

where $\mathcal{R}^{rob}_{0}(f) = \mathcal{R}^{rob}(f)$. We then formulate (\ref{minmax}) as a zero-sum game. In such a game, the learner has a decision set $\{\frac{\partial L^{tr}_{0}(f)}{\partial f},\cdots,\frac{\partial L^{tr}_{K}(f)}{\partial f}\}$, $L^{tr}_{0}(f)$ is the excepted training loss and $L^{tr}_{k}(f)$ is the training loss of class-$k$ for every $1 \leq k \leq K$. The best decision is not known as a priori. 

\begin{remark}
The reason that we add $\frac{\partial L^{tr}_{0}(f)}{\partial f}$ to decision set is the learner can directly choose $\frac{\partial L^{tr}_{0}(f)}{\partial f}$ as a decision in such a game.
\end{remark}

The weight of each decision is initialized as $1/(K+1)$. In epoch $t$, we use the validation set to evaluate the classifier, and use validation loss to denote the cost. The learning rate is $\lambda$. In epoch $t$, the learner updates the model according to the following rule:
\begin{equation}
    f^{t} = f^{t-1} - \lambda \sum_{k=0}^{K} w^{t}_{k} \frac{\partial L^{tr}_{k}(f^{t-1})}{\partial f},
\end{equation}
where  
\begin{equation}
    w^{t}_{k} = \frac{\exp(\sum_{i=1}^{t-1} \eta L^{val}_{k}(f^i) )}{\sum_{k=0}^{K} \exp(\sum_{i=1}^{t-1} \eta L^{val}_{k}(f^i))}.
\end{equation}
After the learner updates the model, it obtains a loss vector from the adversary. The algorithm is described in more detail in Algorithm \ref{alg:bat}. Algorithm \ref{alg:bat} outputs $f^*= \arg \max \limits_{k \in [K]} \min \limits_{f \in \{f^{1},\cdots, f^{T}\}} L^{val}_{k}(f)$. The following theorem provides the guarantee of the worst-class loss.
\begin{theorem}\label{thm:nr}
Assume the range of $L^{val}(f)$ is $[0,1]$, and $1/T \sum_{t=1}^{T} L^{val}_{k}(f^{t}) \geq 1/(1-\eta) \min_{t} L^{val}_{k}(f^{t}) $ for every $k$ and some $\eta \leq 1/2$. We then have
\begin{equation}
    \max_{k} \min_{t} L^{val}_{k}(f^{t}) \leq \frac{1}{T} \sum^{T}_{t=1} \sum^{K}_{k=0} w^{t}_{k}L^{val}_{k}(f^{t}) + \frac{\log (K+1)}{T\eta}. \label{weight}
\end{equation}
\end{theorem}

\begin{proof}
The proof of Theorem \ref{thm:nr} can be found in the Appendix.
\end{proof}

\begin{remark}
Theorem \ref{thm:nr} shows that if we choose a proper $\eta$, after $T$ rounds, the worst-class cost of the best classifier can be bounded by the average loss of previous rounds. Our bound also depends on $\eta$ and $T$; a larger $\eta$ and $T$ will provide a tighter bound.
\end{remark}

\begin{algorithm}[tb]
  \caption{WAT: Worst-class Adversarial Training}
  \label{alg:bat}
\begin{algorithmic}
  \STATE {\bfseries Input:} training data $\mathcal{S}_{tr}$, validation data  $\mathcal{S}_{val}$, learning rate $\lambda$, training epochs $T$, number of classes $K$ and hyper-parameter $\eta$.
  \STATE Initialize $f^{0},w_k^{0} = \frac{1}{K+1}$ for every $k \in [K]$.
  \FOR{$1 \leq t \leq T$}
  \STATE use $\mathcal{S}_{tr}$to obtain $L^{tr}_{0}(f^{t-1}),\cdots, L^{tr}_{K}(f^{t-1})$.
  \STATE use $\mathcal{S}_{val}$ to obtain $L^{val}_{0}(f^{t-1}),  \cdots  , L^{val}_{K}(f^{t-1})$.
  \STATE $f^{t} = f^{t-1} - \lambda \sum_{k=0}^{K} w^{t}_{k} \frac{\partial L^{tr}_{k}(f^{t-1})}{\partial f}$
  \FOR{$0 \leq k \leq K$}
  \STATE $w^{t+1}_{k} = \frac{\exp(\sum_{i=1}^{t} \eta L^{val}_{k}(f^i) )}{\sum_{k=0}^{K} \exp(\sum_{i=1}^{t} \eta L^{val}_{k}(f^i))}$.
  \ENDFOR
  \ENDFOR
  \STATE {\bfseries Output:} $f^*= \arg \max \limits_{k \in [K]} \min \limits_{f \in \{f^{1},\cdots, f^{T}\}} L^{val}_{k}(f)$.
\end{algorithmic}
\end{algorithm}

\section{Generalization Error Bound}
\label{analysis}
This section provides the generalization error bound in terms of the worst-class robust risk. The empirical natural risk and robust risk are defined as $ \mathcal{ \hat R}^{nat}(f) = \frac{1}{n}\sum^{n}_{i=1} \ell(f(\mathbf{x}_i ;\theta), y_i) $ and $ \mathcal{ \hat R}^{rob}(f) = \frac{1}{n}\sum^{n}_{i=1} \max_{\mathbf{x}^{\prime} \in \mathcal{B}(\mathbf{x},\epsilon)} \ell(f(\mathbf{x}^{\prime} ;\theta), y_i) $, respectively.

Rademacher complexity \cite{rademacher} is one of the classic measurements for generalization error. Let $\mathcal{S} = \{\mathbf{z}_1, \mathbf{z}_2, \cdots, \mathbf{z}_n\}$ be an independent and identically distributed (i.i.d.) sample with size $n$ and $\sigma_i$ be a random variable such that $\mathbb{P}[\sigma_i=1] = \mathbb{P}[\sigma_i=-1] = 1/2 $. The Rademacher complexity of function class $\mathcal{H}$ is defined as
$
\mathfrak{R}_{\mathcal{S}}(\mathcal{H}):=\frac{1}{n} \mathbb{E}_{\boldsymbol{\sigma}}\left[\sup _{h \in \mathcal{H}} \sum_{i=1}^{n} \sigma_{i} h\left(\mathbf{z}_{i}\right)\right].
$ We next analyze the gap between the empirical risk and population risk of the worst class. Let the training set $S_k$ be drawn i.i.d. from the distribution $\mathcal{D}_k$. The empirical $k$-th class robust risk is defined as
\begin{equation}
 \mathcal{ \hat R}^{rob}_k(f) = \frac{1}{|\mathcal{S}_k|}\sum_{(\mathbf{x_i},y_i) \in \mathcal{S}_k} \max \limits_{\mathbf{x}^{\prime} \in \mathcal{B}(\mathbf{x},\epsilon)} \ell(f(\mathbf{x_i^\prime} ;\theta), y_i) .
\end{equation}
The empirical worst-class robust risk over $\mathcal{S} := \cup_{k \in [K]} \mathcal{S}_k $ is $ \mathcal{ \hat R}^{rob}_{wc}(f) = \max_{k} \mathcal{ \hat R}^{rob}_k(f) $. and $\tilde{\ell}_{\mathcal{F}}$ is defined as $\ell_{\mathcal{F}} = \{(\mathbf{x},y) \to \ell(f(\mathbf{x}), y) : f \in \mathcal{F} \}$. We assume $\vert \mathcal{S}_k \vert = \vert \mathcal{S} \vert / K$ holds for every $k$. We present the following Theorem.
\begin{theorem}\label{thm:rad}
 Suppose that the range of $\ell(f(\mathbf{x}), y)$ is $[0,B]$. Let $\tilde \ell (f(\mathbf{x}),y) := \max_{\mathbf{x}^{\prime} \in \mathcal{B}(\mathbf{x},\epsilon)} \ell (f(\mathbf{x}^{\prime}),y)$. Then, for any $\delta \in (0,1)$, with probability at least $1-\delta$, the following holds for all $f\in \mathcal{F}$,
\[
\mathcal{ R}^{rob}_{wc}(f) \le \mathcal{ \hat R}^{rob}_{wc}(f) + 2B\max_{k}\mathfrak{R}_{\mathcal{S}_{k}}( \tilde \ell_{\mathcal{F}}) + 3B\sqrt{\frac{K\log \frac{2}{\delta}}{2|\mathcal{S}|}}.
\]
\end{theorem}

\begin{proof}
The proof of Theorem \ref{thm:rad} can be found in the Appendix.
\end{proof}

\subsection{Multi-class Linear Classifiers}
This section studies the generalization error of multi-class linear classifiers. We here consider a $K$-class classification problem. Let $\mathcal{F}_{\mathbf{W}}$ be a multi-class linear classifier hypothesis, and $f_{\mathbf{W}}: X \rightarrow \mathbb{R}^K$ in $\mathcal{F}_{\mathbf{W}}$ be parameterized by a matrix $\mathbf{W}$ with dimension $K \times d$. The $k$-th coordinate of $f_{\mathbf{W}}(\mathbf{x})$ is the score of the $k$-th class, and the prediction of $f_{\mathbf{W}}$ is the class with the highest score among the $K$ classes. Let $\mathbf{w}_k \in \mathbb{R}^d$ be the $k$-th column of $\mathbf{W}^\top$ and be upper bounded by $W$ under the $\ell_p$ norm $(p \geq 1)$: $\mathcal{F}_{\mathbf{W}} = \{ f_{\mathbf{W}}(\mathbf{x}):||\mathbf{W}^\top||_{p,\infty} \leq W  \}$.
For multi-class classification problems, we define the margin operator $\mathcal{M}(\bm{\xi},y): \mathbb{R}^K \times [K] \rightarrow \mathbb{R}$ as $\mathcal{M}(\bm{\xi},y) = \xi_y - \max_{y' \neq y} \xi_{y'}$, and a classifier$f$ predicts correct if and only if $\mathcal{M}(\bm{\xi},y) > 0$. The ramp loss is defined as follows:
\begin{equation}
    \phi_{\gamma}(t)=
    \begin{cases}
    1 & t \leq 0, \\
    1-\frac{t}{\gamma} & 0<t<\gamma, \\
    0 & t \geq \gamma.
    \end{cases}
\end{equation}
Based on the margin operator and ramp loss, we have $ \ell(f_{\mathbf{W}}(\mathbf{x}),y) = \phi_{\gamma}(\mathcal{M}(f_{\mathbf{W}}(\mathbf{x}),y)) $ and $ \tilde \ell(f_{\mathbf{W}}(\mathbf{x}),y) = \max \limits_{\mathbf{x}^{\prime} \in \mathcal{B}(\mathbf{x},\epsilon)} \phi_{\gamma}(\mathcal{M}(\mathbf{f_{\mathbf{W}}}(\mathbf{x}),y)) $. We use $\mathbbm{1}(\cdot)$ to denote a \{0,1\}-valued indicator function. We then present the following Theorem.
\begin{theorem}\label{thm:lin}
Consider the multi-class linear classifiers in the adversarial setting, and suppose that $\frac{1}{p} + \frac{1}{q} = 1$, $p,q \ge 1$. For any fixed $\gamma >0$ and $W > 0$, we have with probability at least $1-\delta$, for all $\mathbf{W}$ such that $\normvec{\mathbf{W}^\top}_{p,\infty} \le W$,
\begin{align*}
1 - Acc^{rob}_{wc}(f,\mathcal{D})
\leq \frac{K}{\vert \mathcal{S} \vert } \sum_{(x_i,y_i) \in \mathcal{S}}  E_{i}
+ \frac{2 W K^3}{\gamma \vert \mathcal{S} \vert} U + c,
\end{align*}
where \\
$E_{i}\!=\!\mathbbm{1}\!\left(\! \left\langle\mathbf{w}_{y_{i}}\!,\! \mathbf{x}_{i}\right\rangle \! \leq \! \gamma \!+ \! \max _{y^{\prime} \neq y_{i}}\left(\left\langle\mathbf{w}_{y^{\prime}}, \mathbf{x}_{i}\right\rangle \! + \!\epsilon \normvec{\mathbf{w}_{y^{\prime}}\!\!-\!\!\mathbf{w}_{y_{i}}}_{1}\right)\!\right),$ \\ 
$c \!=\! \frac{2 W K^2 \epsilon d^{\frac{1}{q}}}{\gamma \sqrt{ \vert \mathcal{S} \vert }}  + 3 \sqrt{\frac{K \log \frac{2}{\delta}}{2 \vert \mathcal{S} \vert }},$\\$U\!=\! \max_{y,k} \mathbb{E}_{\boldsymbol{\sigma}}\left[\left\| \sum_{(\mathbf{x}_i,y_i) \in \mathcal{S}_k} \sigma_{i} \mathbf{\mathbf{x}}_{i} \mathbbm{1}\left(y_{i}=y\right)\right\|_{q}\right].$
\end{theorem}

\begin{proof}
The proof of Theorem \ref{thm:lin} can be found in the Appendix.
\end{proof}

\begin{remark}
Only if we optimize worst-class robust risk, as in our method, Theorem 2 and 3 hold. However, previous works do not optimize this risk and Theorem 2 and 3 are not applicable to them.
\end{remark}

\section{Experiments}
\label{experiment}

\begin{table*}[h!]
\caption{ Comparison results of all methods using ResNet-18 on CIFAR-10 and CIFAR-100. We evaluate every method in terms of both accuracy (\%) and $\rho$. We report the average natural accuracy, worst-class natural accuracy, average robust accuracy, worst-class robust accuracy,  $\rho_{nat}$, $\rho_{pgd}$, $\rho_{cw}$ and $\rho_{AA}$ for every method. We use \textbf{bold} to denote the best value in every metric. }
\begin{center}
\setlength{\tabcolsep}{2mm}{
\begin{tabular}{lcccccccccccc}
\toprule
CIFAR-10 & \multicolumn{3}{c}{Natural} & \multicolumn{3}{c}{PGD-100}  & \multicolumn{3}{c}{CW}  & \multicolumn{3}{c}{AutoAttack}\\
\midrule
Method & Avg. & Wst. & $\rho_{nat}$ & Avg. & Wst. & $\rho_{pgd}$  & Avg. & Wst.  & $\rho_{cw}$  & Avg. & Wst. & $\rho_{AA}$\\
\midrule
TRADES      & \textbf{82.11} & 64.6 & 0 & \textbf{51.69} & 25.2 & 0 & 50.38 & 24.1 & 0 & \textbf{48.64} & 21.7 & 0\\
FRL-RW      & 81.75 & 69.2 & \textbf{0.067} & 49.02  & 30.8 & 0.171 & 47.80 & 27.8 & 0.102 & 46.08 & 25.4  & 0.118\\
FRL-RWRM    & 80.69 & \textbf{71.4} & 0.088 & 49.16 & 32.0 & 0.221  & 47.45 & 28.1 & 0.108 & 45.94 & 26.1 & 0.147\\
CSL         & 76.29 & 67.1 & -0.032 & 43.30 & 33.8 & 0.179 & 41.60 & 31.3 & 0.124 & 40.32 & 29.2 & 0.175\\
\midrule
Ours        & 80.98 & 69.5 & 0.062 & 49.13 & \textbf{36.6} & \textbf{0.403} & 47.57 & \textbf{33.3} & \textbf{0.326} & 46.04 & \textbf{30.1} & \textbf{0.334}\\
\bottomrule
\end{tabular}
}
\end{center}
\begin{center}

\setlength{\tabcolsep}{2mm}{
\begin{tabular}{lcccccccccccc}
\toprule
CIFAR-100 & \multicolumn{3}{c}{Natural} & \multicolumn{3}{c}{PGD-100}  & \multicolumn{3}{c}{CW}  & \multicolumn{3}{c}{AutoAttack}\\
\midrule
Method & Avg. & Wst. & $\rho_{nat}$ & Avg. & Wst. & $\rho_{pgd}$  & Avg. & Wst.  & $\rho_{cw}$  & Avg. & Wst. & $\rho_{AA}$\\
\midrule
TRADES      & \textbf{54.57}    & 19.00  & 0 & \textbf{27.39} & 3.00 & 0 & \textbf{24.87} & 1.00 & 0 & \textbf{23.57} & 1.00 & 0\\
FRL-RW      &  53.08 & \textbf{24.00} & \textbf{0.236} & 25.76 & 3.00 & -0.060 & 22.39 & 2.00 & 0.900 & 21.09 & 1.00 & -0.105 \\
FRL-RWRM    & 52.55  & 22.00 & 0.121 & 26.04 & 4.00 & 0.284 & 22.33 & 2.00 & 0.898 & 21.11 & 2.00 & 0.896 \\
CSL         & 53.83 & 21.00 & 0.092  & 26.19 & 4.00 & 0.290 & 22.35 & 2.00 & 0.899 & 22.25 & 2.00 & 0.944 \\
\midrule
Ours        & 53.99 & 19.00 & -0.020 & 26.91 & \textbf{5.00} & \textbf{0.643} & 24.26 & \textbf{3.00}  & \textbf{1.945} & 22.89 & \textbf{3.00} & \textbf{1.971 }\\
\bottomrule
\end{tabular}
}
\end{center}
\label{tab:main}
\end{table*}

\begin{figure*}[h!]
\centering

\subfigure[Ours vs TRADES]{
    \begin{minipage}[t]{0.22\linewidth}
    \centering
    \includegraphics[width=0.85\linewidth]{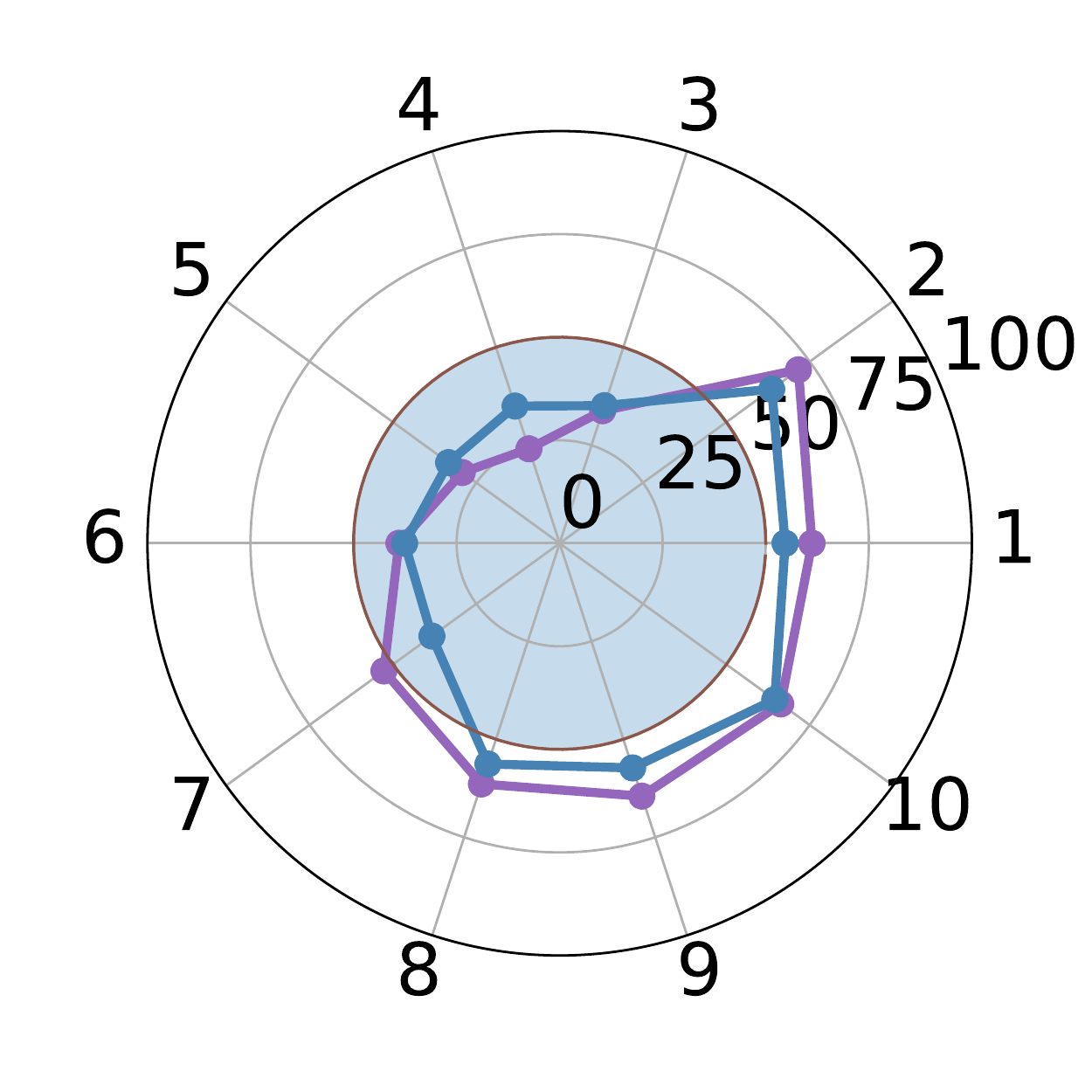}
    \label{fig:vs:trades}
    \end{minipage}
}
\subfigure[Ours vs CSL]{
    \begin{minipage}[t]{0.22\linewidth}
    \centering
    \includegraphics[width=0.85\linewidth]{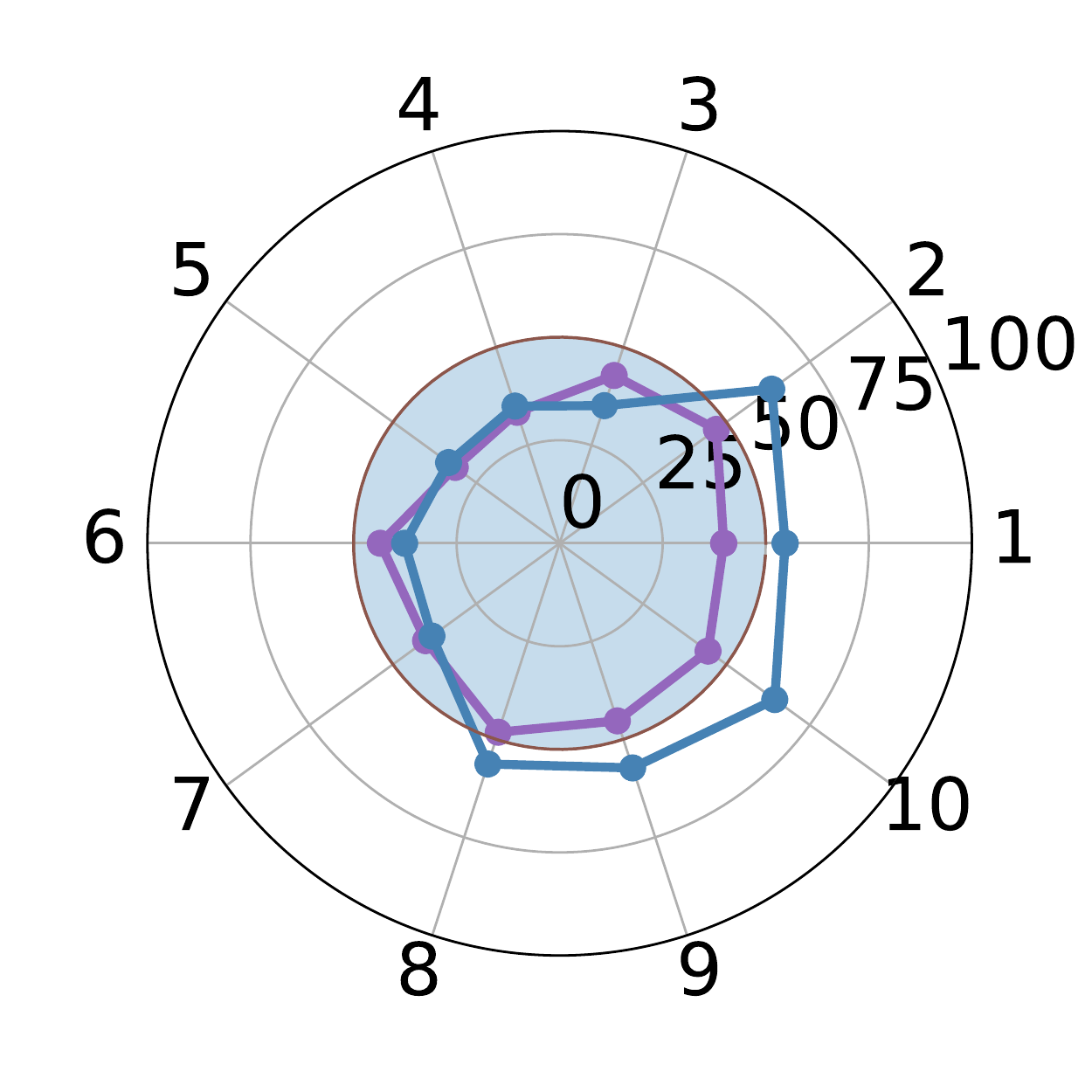}
    \label{fig:vs:csl}
    \end{minipage}
}
\subfigure[Ours vs FRL-RW]{
    \begin{minipage}[t]{0.22\linewidth}
    \centering
    \includegraphics[width=0.85\linewidth]{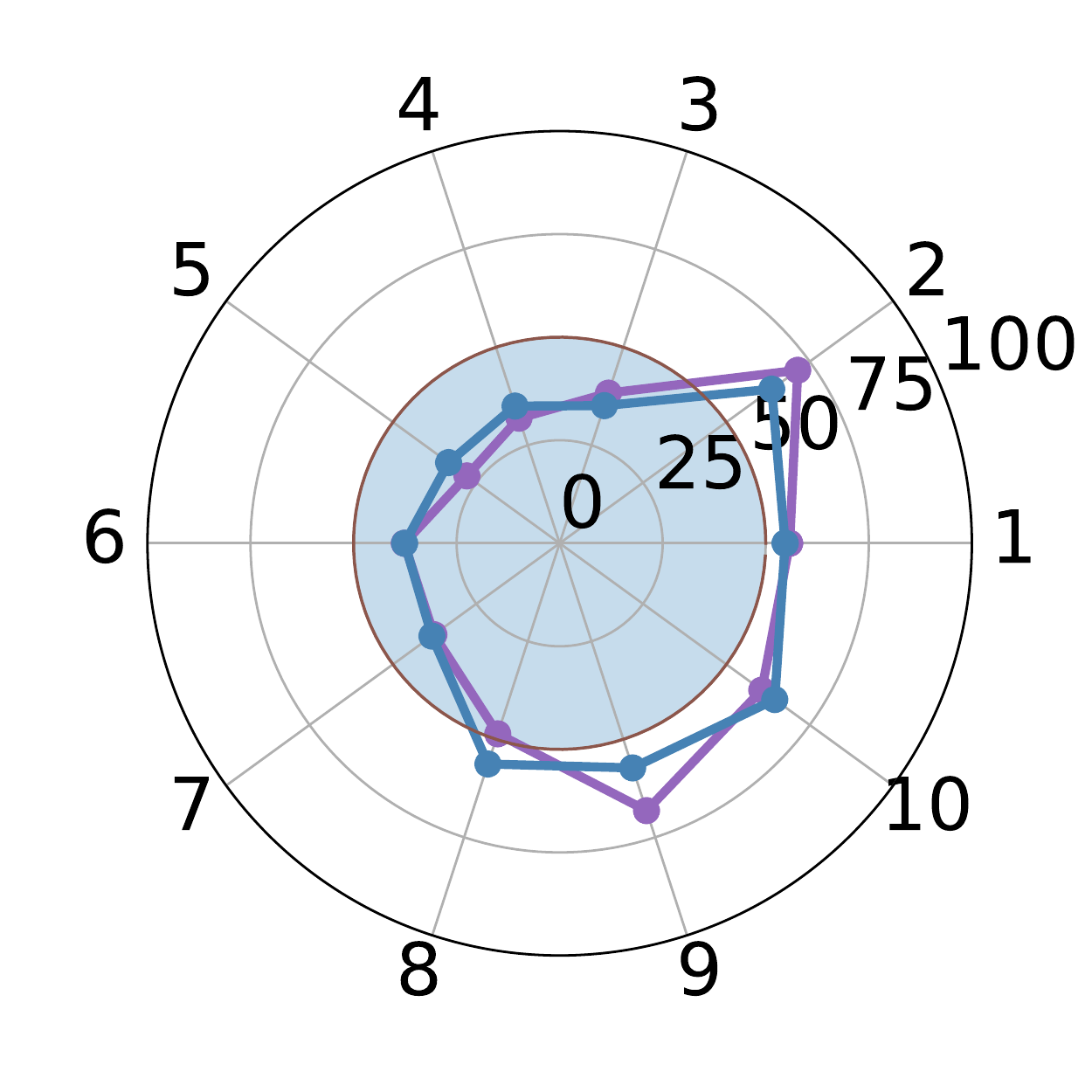}
    \label{fig:vs:rw}
    \end{minipage}
}
\subfigure[Ours vs FRL-RWRM]{
    \begin{minipage}[t]{0.22\linewidth}
    \centering
    \includegraphics[width=0.85\linewidth]{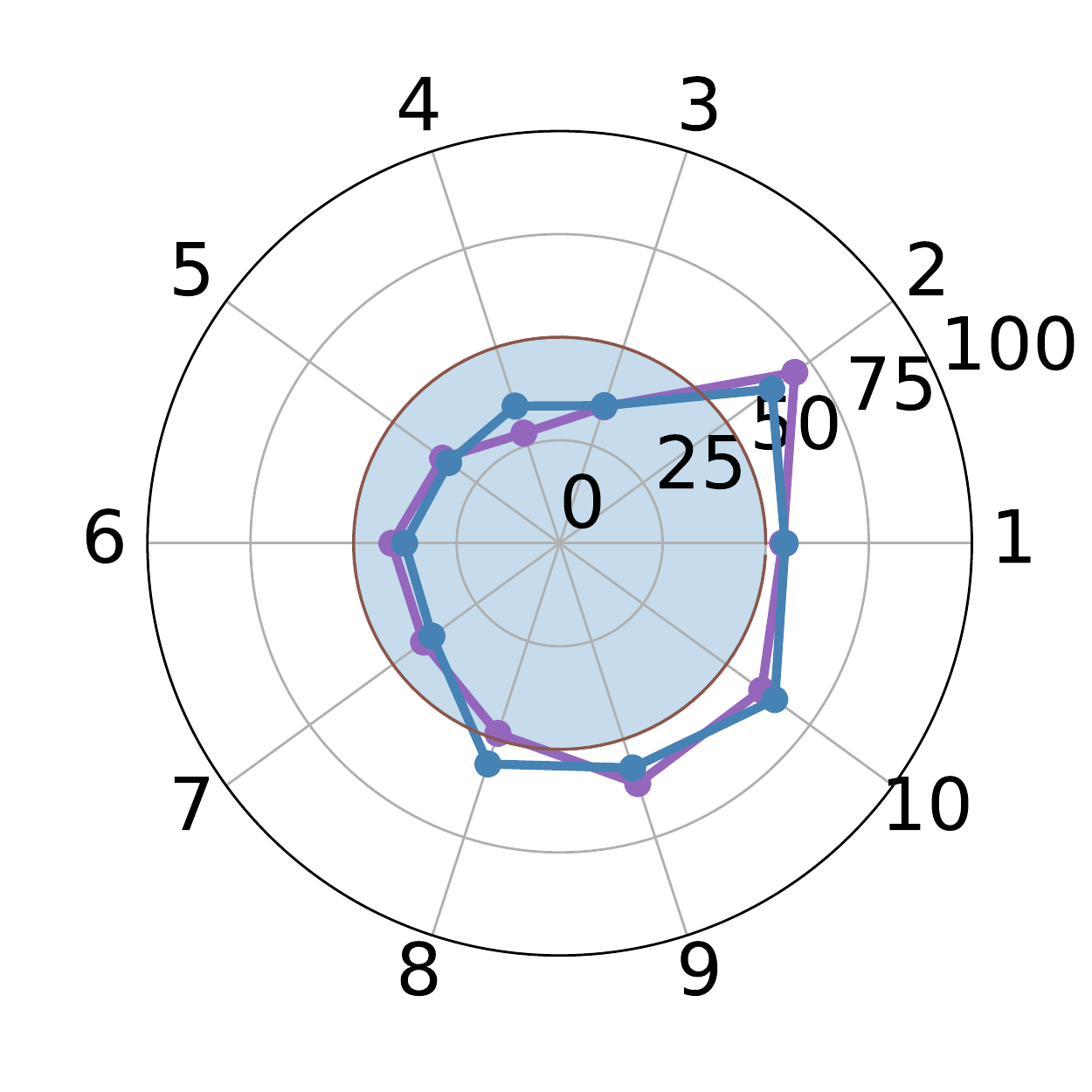}
    \label{fig:vs:rwrm}
    \end{minipage}
}
\caption{ Class-wise robust accuracy disparity of all methods using ResNet-18 on CIFAR-10. We compare our method and another method in terms of the class-wise robust accuracy evaluated under CW attack. We denote the results of our method with a blue line, while the results of the comparison methods are represented by a purple line.}
\label{fig:compare}

\end{figure*}

In this section, we conduct experiments on various datasets and models to evaluate the performance of our proposed method. Code is available at https://github.com/boqili/WAT.

\subsection{Datasets and Baselines}
The datasets used in the experiments are CIFAR-10 and CIFAR-100 \cite{cifar}, which are described in more detail in the Appendix.

\citet{trades}, \citet{frl} and \citet{analysis:benz} are used as our baselines. \textbf{TRADES} \cite{trades} is one of the most popular adversarial training methods. $\textbf{FRL}$ is presented in \citet{frl}. $\textbf{FRL}$ has two variants: $\textbf{FRL-RW}$ is based on the re-weight strategy, and $\textbf{FRL-RWRM}$ is based on the re-weight and re-margin strategy. $\textbf{Cost-sensitive Learning (CSL)}$ \cite{analysis:benz} is a classical approach to solving the class-imbalanced problem on imbalanced datasets \cite{csl:ting,csl:khan}. To be fair, we use the same hyper-parameters and perform the model selection for each method.

\subsection{Evaluations}
We use the following measures to evaluate the performance of all methods.

$\textbf{Average and Worst-class accuracy}$. Following \cite{frl}, we use average natural accuracy, average robust accuracy, worst-class natural accuracy and worst-class robust accuracy to evaluate the performance of all methods. We use three strong adversarial attacks PGD-100, CW\cite{cw} attack and AutoAttack\cite{aa} to evaluate robust accuracy. We set perturbation radius $\epsilon=8/255$ for CIFAR-10 and CIFAR-100. Other details can be found in the Appendix. 

$\textbf{Class-wise Variance ($CV$)}$. Class-wise variance is a common measure used in \cite{frl} and \cite{analysis:tian}. The definition of $CV$ given in \cite{analysis:tian} is presented below.
\begin{definition} \cite{analysis:tian}
Given one dataset containing C classes, the accuracy of each class $c$ is $a_{c}$, the average accuracy over all class is $\bar{a} \!=\! \frac{1}{C}\sum_{c=1}^{C} a_{c}$, and the $CV$ is defined as:  $CV \!=\! \frac{1}{c} \sum^{C}_{c=1}(a_{c} - \bar{a})^2$.

\end{definition}

 We use $CV_{nat}$ to denote the class-wise variance of natural accuracy and $CV_{rob}$ to denote the class-wise variance of robustness accuracy, We also use $\rho$ as defined in Eq.(\ref{eq:rho}) to evaluate the method in terms of both the average and worst-class accuracies.

\subsection{Results}

In Table \ref{tab:main}, we report the performance of every method using ResNet-18 on CIFAR-10 and CIFAR-100. We can clearly observe that our method successfully outperforms other methods on both CIFAR-10 and CIFAR-100. More specifically, under PGD-100 attack, our method improves the worst-class robust accuracy of all compared methods by at least 2.8\% on the CIFAR-10 dataset and 1.0\%  on the CIFAR-100 dataset, while improving the worst-class robust accuracy of all compared methods for at least 2.0\%  on CIFAR-10 dataset and 1.0\%  on CIFAR-100 under CW attack. Under AutoAttack, our method improves the worst-class robust accuracy of all compared methods by at least 0.9\% on the CIFAR-10 dataset and 1.0\%  on the CIFAR-100 dataset as well. Moreover, compared with TRADES, although all compared methods increase the robust accuracy, our method achieves the best $\rho_{pgd}$, $\rho_{cw}$ and $\rho_{AA}$ value; in short, we sacrifice the least average robust accuracy to obtain the highest worst-class robust accuracy. 

\begin{table*}[ht]

\caption{Comparison results of all methods using WideResNet-34-10 on CIFAR-10.}

\begin{center}
\setlength{\tabcolsep}{2mm}{
\begin{tabular}{lccccccccccccc}
\toprule
CIFAR-10 & \multicolumn{3}{c}{Natural} & \multicolumn{3}{c}{PGD-100}  & \multicolumn{3}{c}{CW}  & \multicolumn{3}{c}{AutoAttack}\\
\midrule
Method & Avg. & Wst. & $\rho_{nat}$ & Avg. & Wst. & $\rho_{pgd}$  & Avg. & Wst.  & $\rho_{cw}$  & Avg. & Wst. & $\rho_{AA}$\\
\midrule
TRADES       & \textbf{84.51}  &  64.7 & 0 & \textbf{53.68} & 23.3 & 0 & \textbf{53.18} & 22.8 & 0  & \textbf{51.22} & 20.9  & 0\\
FRL-RW     & 83.93 & 74.5 & \textbf{0.145} & 50.59  & 30.0 & 0.230  & 50.58 & 29.1 & 0.227 & 48.36 & 27.1 & 0.241 \\
FRL-RWRM     & 83.86 & 72.1 & 0.107 & 51.25  & 32.9 & 0.367 & 51.08 & 32.2 & 0.373 & 48.98 & 28.6 & 0.325 \\
CSL         & 79.78 & \textbf{75.1} & 0.105 & 45.7 & 32.2 & 0.233 & 44.74 & 30.8 & 0.192 & 43.10 & 29.4 & 0.248 \\
\midrule
Ours        & 83.71 & 74.0 & 0.062 & 51.53 & \textbf{34.9} & \textbf{0.458} & 50.89 & \textbf{33.4} & \textbf{0.422} & 49.12 & \textbf{30.7} & \textbf{0.428} \\
\bottomrule
\end{tabular}
}
\end{center}
\label{tab:wrn}

\end{table*}

\begin{table*}[h]
\caption{ Results of our method with different $\eta$ using ResNet-18 on CIFAR-10.}

\begin{center}
\setlength{\tabcolsep}{2mm}{

\begin{tabular}{lccccccccccccc}
\toprule
CIFAR-10 & \multicolumn{3}{c}{Natural} & \multicolumn{3}{c}{PGD-100}  & \multicolumn{3}{c}{CW}  & \multicolumn{3}{c}{AutoAttack}\\
\midrule
Method & Avg. & Wst. & $\rho_{nat}$ & Avg. & Wst. & $\rho_{pgd}$  & Avg. & Wst.  & $\rho_{cw}$  & Avg. & Wst. & $\rho_{AA}$\\
\midrule
TRADES      & \textbf{82.11}  &  64.6 & 0  & \textbf{51.69} & 25.2 & 0 & \textbf{50.38} & 24.1 & 0 & \textbf{48.64} & 21.7 & 0\\
Ours($\eta$=0.01)  & 81.54 & 68.0 & 0.046 & 50.50 & 26.6 & 0.033 & 49.86 & 25.0 & 0.027 & 47.65 & 22.6 & 0.021 \\
Ours($\eta$=0.05)    & 81.76 & 69.3 & \textbf{0.068} & 50.06 & 34.2 & 0.326 & 49.53 & 31.7 & 0.298 & 47.05 & 28.1 & 0.262 \\
Ours($\eta$=0.1)    & 80.98 & \textbf{69.5} & 0.062 & 49.13 & 36.6 & 0.403 & 47.57 & \textbf{33.3} & \textbf{0.326} & 46.04 & 30.1 & 0.334 \\
Ours($\eta$=0.5)    & 79.30 & 67.3 & 0.008 & 48.09 & \textbf{37.5} & \textbf{0.418} & 45.42 & 32.5 & 0.250 & 43.98 & \textbf{31.1} & \textbf{0.337} \\
\bottomrule
\end{tabular}
}
\end{center}
\label{tab:eta}
\end{table*}

Furthermore, to study the effectiveness of our method in more detail, we conduct a comparison of the class-wise robust accuracy evaluated under CW attack between our method and all compared methods in Figure \ref{fig:compare}. As shown in Figure \ref{fig:vs:trades}, our method achieves higher robust accuracy of class-4 and class-5 than TRADES, thus, our method obtains a good performance on worst-class robust accuracy. In Figure \ref{fig:vs:csl}, although CSL achieves a great performance on the worst class, it performs worse than our method on most other classes, which leads to a low average robust accuracy. From Figures \ref{fig:vs:rw} and \ref{fig:vs:rwrm}, we can see that our method achieves higher robust accuracy on class-4 (the most vulnerable class) than the other two baselines. Moreover, our proposed method significantly outperforms the other two baselines on class-5 and class-8, which contributes to the highest $\rho_{cw}$ of our method. The results of class-wise robust accuracy disparity of all the methods evaluated under PGD-100 attack and AutoAttack on CIFAR-10 can be found in the Appendix.

We go on to evaluate the performance of all the methods on WideResNet-34-10\cite{wideresnet}. The experimental results can be found in Table \ref{tab:wrn}. From the results in  Table \ref{tab:wrn}, we can find that our method achieves the highest worst-class robust accuracy evaluated under all three attacks with at least 1.3\% improvement. we also achieve the highest $\rho_{pgd}$, $\rho_{cw}$ and $\rho_{AA}$ while we have comparable result with compared methods in average robust accuracy evaluated under all three attacks on CIFAR-10.

\subsection{Parameter Analysis on $\eta$}

We study the impact of hyper-parameter $\eta$ used in our method on average and worst-class robust accuracy. We vary the hyper-parameter $\eta$ from \{0.01,0.05,0.1,0.5\}, and show the results in Table \ref{tab:eta}. We find that a trade-off between the average robust accuracy and the worst-class robust accuracy exists, and if we improve the average robust accuracy, the worst-class robust accuracy decreases at the same time. However, a larger $\eta$ does not lead to a larger $\rho_{nat}$ and $\rho_{cw}$ . In our experiments, we find $\eta=0.1$ yields the best $\rho_{nat}$ and $\rho_{cw}$ while $\eta=0.5$ yields the best $\rho_{pgd}$ and $\rho_{AA}$.

\begin{table}[h]
\caption{Comparison results between $CV_{cw}$ and $\rho_{cw}$ using ResNet-18 on CIFAR-10.}
\label{tab:variance}
\begin{center}
\setlength{\tabcolsep}{2mm}{
\begin{tabular}{lcccc}

\toprule
CIFAR-10 & \multicolumn{4}{c}{CW Attack} \\
\midrule
Method & Avg. & Wst. & $CV_{cw}$ & $\rho_{cw}$ \\
\midrule
TRADES      & \textbf{50.38} & 24.1 & 0.0269  & 0  \\
\midrule
FRL-RW      & 47.80 & 27.8 & 0.0215  & 0.102  \\
FRL-RWRM    & 47.45 & 28.1 & 0.0172  & 0.108  \\
CSL         & 41.60 & 31.3 & \textbf{0.0027}  & 0.124  \\
Ours        & 47.57 & \textbf{33.3} & 0.0147  & \textbf{0.326}  \\
\bottomrule
\end{tabular}
}
\end{center}
\label{tab:rho}

\end{table}

\subsection{Comparison between $CV$ and $\rho$}
\label{sec:rho}

From the results in Table \ref{tab:variance}, we can see that CSL obtains the lowest $CV_{cw}$ value, while the average robust accuracy of CSL is the worst. Notably, $CV_{cw}$ is not a good measurement because it does not consider the trade-off between average and worst-class robust accuracy. From the results in Table \ref{tab:variance}, we can also see that our method achieves the best $\rho_{cw}$, has the highest worst-class robust accuracy, and is comparable with FRL and CSL in average robust accuracy. Therefore, $\rho_{cw}$ is a more reasonable measurement than $CV_{cw}$ because it considers average robust accuracy and worst-class robust accuracy at the same time. The results evaluated under PGD-100 attack and AutoAttack are shown in the Appendix.

\section{Conclusion}
\label{conclusion}
To improve the worst-class robustness in adversarial training, this paper proposes a novel framework of worst-class adversarial training and leverages no-regret dynamics to solve the problem. Theoretically, we provide the guarantee of the worst-class loss and analyze the generalization error bound in terms of the worst-class robust risk based on Rademacher complexity. Moreover, we propose a measurement to evaluate the method in terms of both the average and worst-class accuracies. Empirical results verify the superiority of our proposed approach.

\section{Acknowledgments}
This work is supported by the National Natural Science Foundation of China under Grant 61976161.

\bibliography{aaai23}

\newpage
\appendix
\onecolumn

\section{Proof of Theorem}
\label{app:thm}

\subsection{Proof of Theorem \ref{thm:nr}}
\label{app:nr}

\begin{customthm}{1}
\label{appendix:thm:nr}
Assume the range of $L^{val}(f)$ is $[0,1]$, and $1/T \sum_{t=1}^{T} L^{val}_{k}(f^{t}) \geq 1/(1-\eta) \min_{t} L^{val}_{k}(f^{t}) $ for every $k$ and some $\eta \leq 1/2$. We then have
\begin{equation}
    \max_{k} \min_{t} L^{val}_{k}(f^{t}) \leq \frac{1}{T} \sum^{T}_{t=1} \sum^{K}_{k=0} w^{t}_{k}L^{val}_{k}(f^{t}) + \frac{\log (K+1)}{T\eta}. \label{appendix:weight}
\end{equation}
\end{customthm}

We have this no-regret bound as a Lemma from \cite{mwua}.
\begin{lemma}\label{lemma:nr}
Assume that all cost $C^{t}_{i} \in  [-1,1]$ and $\eta \leq 1/2$. Then the Multiplicative Weights algorithm guarantees that after T rounds, for any $k$, we have
\begin{equation*}
    \sum^{T}_{t=1} \sum^{K}_{k=1} C^{t}_{k} \cdot p^{t}_{k} \geq \sum_{t=1}^{T} C^{t}_{k} -\eta \sum_{t=1}^{T} \vert C^{t}_{k} \vert - \frac{\log K}{\eta}.
\end{equation*}
\end{lemma}
Now we prove Theorem \ref{appendix:thm:nr}.
\begin{proof}
From Lemma \ref{lemma:nr}, for every $k$,  we have
\begin{equation}
    \sum^{T}_{t=1} L^{val}_{k}(f^{t}) - \eta \sum^{T}_{t=1} \vert L^{val}_{k}(f^{t}) \vert \leq \sum^{T}_{t=1} \sum^{K}_{k=0} w^{t}_{k} L^{val}_{k}(f^{t}) + \frac{\log (K+1)}{\eta}.
\end{equation}
Use the assumption that the range of $L(f)$ is $[0,1]$, for every $k$ we can yield
\begin{equation}  \label{thm3:ineq1}
    (1-\eta)\sum^{T}_{t=1}L^{val}_{k}(f^{t}) \leq \sum^{T}_{t=1} \sum^{K}_{k=0} w^{t}_{k}L^{val}_{k}(f^{t}) + \frac{\log (K+1)}{\eta}.
\end{equation}

Because inequation (\ref{thm3:ineq1}) holds for every $k$, with the assumption that $1/T \sum_{t=1}^{T} L^{val}_{k}(f^{t}) \geq 1/(1-\eta) \min_{t} L^{val}_{k}(f^{t}) $ holds for every $k$ and some $\eta \leq 1/2$, we can yield
\begin{equation}
    \min_{t} L^{val}_{k}(f^{t}) \leq \frac{1-\eta}{T}  \sum^{T}_{t=1}L^{val}_{k}(f^{t}) \leq \frac{1}{T} \sum^{T}_{t=1} \sum^{K}_{k=0} w^{t}_{k}L^{val}_{k}(f^{t}) + \frac{\log (K+1)}{T\eta},
\end{equation}
for every $k$. Thus we have
\begin{equation}
    \max_k \min_{t} L^{val}_{k}(f^{t}) \leq \frac{1}{T} \sum^{T}_{t=1} \sum^{K}_{k=0} w^{t}_{k}L^{val}_{k}(f^{t}) + \frac{\log (K+1)}{T\eta}.
\end{equation}

We conclude this proof.
\end{proof}

\subsection{Proof of Theorem \ref{thm:rad}}
\label{app:rad}

\begin{customthm}{2}
\label{appendix:thm:rad}
 Suppose that the range of $\ell(f(\mathbf{x}), y)$ is $[0,B]$. Let $\tilde \ell (f(\mathbf{x}),y) := \max_{\mathbf{x}^{\prime} \in \mathcal{B}(\mathbf{x},\epsilon)} \ell (f(\mathbf{x}^{\prime}),y)$. Then, for any $\delta \in (0,1)$, with probability at least $1-\delta$, the following holds for all $f\in \mathcal{F}$,
\[
\mathcal{ R}^{rob}_{wc}(f) \le \mathcal{ \hat R}^{rob}_{wc}(f) + 2B\max_{k}\mathfrak{R}_{\mathcal{S}_{k}}( \tilde \ell_{\mathcal{F}}) + 3B\sqrt{\frac{K\log \frac{2}{\delta}}{2|\mathcal{S}|}}.
\]
\end{customthm}

To prove Theorem \ref{appendix:thm:rad}, we need this following lemma.
\begin{lemma}\label{lemma:rad} \cite{theory:yin}
 Suppose that the range of $\ell(f(x), y)$ is $[0,B]$. Let $\tilde \ell (f(x),y) := max_{x' \in \mathcal{B}(x,\epsilon)} \ell (f(x'),y)$. Then, for any $\delta \in (0,1)$, with probability at least $1-\delta$, the following holds for all $f\in \mathcal{F}$,
\[
\mathcal{ R}^{rob}(f) \le \mathcal{ \hat R}^{rob}(f) + 2B\mathfrak{R}_{\mathcal{S}}( \tilde \ell_{\mathcal{F}}) + 3B\sqrt{\frac{\log \frac{2}{\delta}}{2\vert \mathcal{S} \vert }}.
\]
\end{lemma}

Now we prove Theorem \ref{appendix:thm:rad}.

\begin{proof}
According to Lemma \ref{lemma:rad} , we have following results for every $k$ under the same assumption.
\begin{equation}
\mathcal{ R}^{rob}_k(f) \le \mathcal{ \hat R}^{rob}_k(f) + 2B\mathfrak{R}_{\mathcal{S}_k}( \tilde \ell_{\mathcal{F}}) + 3B\sqrt{\frac{\log \frac{2}{\delta}}{2\vert \mathcal{S}_k \vert}}.
\end{equation}
Take the maximum values with $k$ at left hand and right hand of the inequation respectively, we have
\begin{equation}
\max_k \mathcal{ R}^{rob}_k(f) \le \max_k \left[\mathcal{ \hat R}^{rob}_k(f) + 2B\mathfrak{R}_{\mathcal{S}_k}( \tilde \ell_{\mathcal{F}}) + 3B\sqrt{\frac{\log \frac{2}{\delta}}{2|\mathcal{S}_k|}}\right].
\end{equation}
Use equation (\ref{ineq:1})
\begin{equation} \label{ineq:1}
\max_x (f(x) + g(x)) \leq \max_x f(x) + \max_x g(x),
\end{equation}
we have
\begin{equation}
\mathcal{ R}^{rob}_{wc}(f) \le \mathcal{ \hat R}^{rob}_{wc}(f) + 2B\max_{k}\mathfrak{R}_{\mathcal{S}_{k}}( \tilde \ell_{\mathcal{F}}) + 3B\sqrt{\frac{K\log \frac{2}{\delta}}{2|\mathcal{S}|}}.
\end{equation}
We conclude this proof.
\end{proof}

\subsection{Proof of Theorem \ref{thm:lin}}
\label{app:lin}

\begin{customthm}{3}\label{appendix:thm:lin}
Consider the multi-class linear classifiers in the adversarial setting, and suppose that $\frac{1}{p} + \frac{1}{q} = 1$, $p,q \ge 1$. For any fixed $\gamma >0$ and $W > 0$, we have with probability at least $1-\delta$, for all $\mathbf{W}$ such that $\normvec{\mathbf{W}^\top}_{p,\infty} \le W$,
\begin{align*}
1 - Acc^{rob}_{wc}(f,\mathcal{D})
\leq \frac{K}{\vert \mathcal{S} \vert } \sum_{(x_i,y_i) \in \mathcal{S}}  E_{i}
+ \frac{2 W K^3}{\gamma \vert \mathcal{S} \vert} U + c,
\end{align*}
where
\begin{equation}
\begin{aligned}
&E_{i}\!=\!\mathbbm{1}\!\left(\! \left\langle\mathbf{w}_{y_{i}}\!,\! \mathbf{x}_{i}\right\rangle \! \leq \! \gamma \!+ \! \max _{y^{\prime} \neq y_{i}}\left(\left\langle\mathbf{w}_{y^{\prime}}, \mathbf{x}_{i}\right\rangle \! + \!\epsilon \normvec{\mathbf{w}_{y^{\prime}}\!\!-\!\!\mathbf{w}_{y_{i}}}_{1}\right)\!\right), \\
&U\!=\! \max_{y,k} \mathbb{E}_{\boldsymbol{\sigma}}\left[\left\| \sum_{(\mathbf{x}_i,y_i) \in \mathcal{S}_k} \sigma_{i} \mathbf{\mathbf{x}}_{i} \mathbbm{1}\left(y_{i}=y\right)\right\|_{q}\right], \\
&c \!=\! \frac{2 W K^2 \epsilon d^{\frac{1}{q}}}{\gamma \sqrt{ \vert \mathcal{S} \vert }}  + 3 \sqrt{\frac{K \log \frac{2}{\delta}}{2 \vert \mathcal{S} \vert }}.
\end{aligned}
\end{equation}
\end{customthm}

To prove Theorem \ref{appendix:thm:lin}, we need this following lemma.
\begin{lemma}\label{lemma:multi-class:rad} \cite{theory:yin}
Consider the multi-class linear classifiers in the adversarial setting, and suppose that $\frac{1}{p} + \frac{1}{q} = 1$, $p,q \ge 1$. For any fixed $\gamma >0$ and $W > 0$, we have with probability at least $1-\delta$, for all $\mathbf{W}$ such that $||\mathbf{W}^\top||_{p,\infty} \le W$,

\begin{align*}
& \mathbb{P}_{(\mathbf{x}, y) \sim \mathcal{D}}\left\{\exists \mathbf{x}^{\prime} \in \mathcal{B}(\mathbf{x},\epsilon), \text { s.t. } y \neq \arg \max _{y^{\prime} \in[K]}\left\langle\mathbf{w}_{y^{\prime}}, \mathbf{x}\right\rangle\right\} \\
\leq & \frac{1}{n} \sum_{i=1}^{n} \mathbbm{1}\left(\left\langle\mathbf{w}_{y_{i}}, \mathbf{x}_{i}\right\rangle \leq \gamma+\max _{y^{\prime} \neq y_{i}}\left(\left\langle\mathbf{w}_{y^{\prime}}, \mathbf{x}_{i}\right\rangle+\epsilon\left\|\mathbf{w}_{y^{\prime}}-\mathbf{w}_{y_{i}}\right\|_{1}\right)\right) \\
+ & \frac{2 W K}{\gamma}\left[\frac{\epsilon \sqrt{K} d^{\frac{1}{q}}}{\sqrt{n}}+\frac{1}{n} \sum_{y=1}^{K} \mathbb{E}_{\boldsymbol{\sigma}}\left[\left\|\sum_{i=1}^{n} \sigma_{i} \mathbf{x}_{i} \mathbbm{1}\left(y_{i}=y\right)\right\|_{q}\right] \right]+3 \sqrt{\frac{\log \frac{2}{\delta}}{2 n}}.
\end{align*}

\end{lemma}

Now we prove Theorem \ref{appendix:thm:lin}.
\begin{proof}
According to Lemma \ref{lemma:multi-class:rad} , we have following results for every $k$ under the same assumption.

\begin{align} \label{lem:1}
& \mathbb{P}_{(\mathbf{x}, y) \sim \mathcal{D}_k}\left\{\exists \mathbf{x}^{\prime} \in \mathcal{B}(\mathbf{x},\epsilon), \text { s.t. } y \neq \arg \max _{y^{\prime} \in[K]}\left\langle\mathbf{w}_{y^{\prime}}, \mathbf{x}\right\rangle\right\} \notag \\
\leq & \frac{1}{|{\mathcal{S}_k}|}\sum_{(x_i,y_i) \in \mathcal{S}_k}  \mathbbm{1}\left(\left\langle\mathbf{w}_{y_{i}}, \mathbf{x}_{i}\right\rangle \leq \gamma+\max _{y^{\prime} \neq y_{i}}\left(\left\langle\mathbf{w}_{y^{\prime}}, \mathbf{x}_{i}\right\rangle+\epsilon\left\|\mathbf{w}_{y^{\prime}}-\mathbf{w}_{y_{i}}\right\|_{1}\right)\right)  \\
+ & \frac{2 W K}{\gamma}\left[\frac{\epsilon \sqrt{K} d^{\frac{1}{q}}}{\sqrt{|{\mathcal{S}_k}|}}+\frac{1}{|{\mathcal{S}_k}|} \sum_{y=1}^{K} \mathbb{E}_{\boldsymbol{\sigma}}\left[\left\| \sum_{(x_i,y_i) \in \mathcal{S}_k}  \sigma_{i} \mathbf{x}_{i} \mathbbm{1}\left(y_{i}=y\right)\right\|_{q}\right] \right]+3 \sqrt{\frac{\log \frac{2}{\delta}}{2 |{\mathcal{S}_k}|}}. \notag
\end{align}

Take the maximum values with $k$ at left hand and right hand of (\ref{lem:1}) respectively and use (\ref{ineq:1}), we can yield
\begin{align} \label{thm2:1}
& \max_k \mathbb{P}_{(\mathbf{x}, y) \sim \mathcal{D}_k}\left\{\exists \mathbf{x}^{\prime} \in \mathcal{B}(\mathbf{x},\epsilon), \text { s.t. } y \neq \arg \max _{y^{\prime} \in[K]}\left\langle\mathbf{w}_{y^{\prime}}, \mathbf{x}\right\rangle\right\} \notag \\
\leq & \max_k \left[ \frac{1}{|\mathcal{S}_k|}\sum_{(x_i,y_i) \in \mathcal{S}_k}  \mathbbm{1}\left(\left\langle\mathbf{w}_{y_{i}}, \mathbf{x}_{i}\right\rangle \leq \gamma+\max _{y^{\prime} \neq y_{i}}\left(\left\langle\mathbf{w}_{y^{\prime}}, \mathbf{x}_{i}\right\rangle+\epsilon\left\|\mathbf{w}_{y^{\prime}}-\mathbf{w}_{y_{i}}\right\|_{1}\right)\right) \right] \\
+ & \max_k \left[ \frac{2 W K}{\gamma}\left[\frac{\epsilon \sqrt{K} d^{\frac{1}{q}}}{\sqrt{|\mathcal{S}_k|}}+\frac{1}{|\mathcal{S}_k|} \sum_{y=1}^{K} \mathbb{E}_{\boldsymbol{\sigma}}\left[\left\| \sum_{(x_i,y_i) \in \mathcal{S}_k}  \sigma_{i} \mathbf{x}_{i} \mathbbm{1}\left(y_{i}=y\right)\right\|_{q}\right] \right]+3 \sqrt{\frac{\log \frac{2}{\delta}}{2 |\mathcal{S}_k|}}  \right]. \notag
\end{align}
Because for every $k$ we have
\begin{equation}
    \sum_{(x_i,y_i) \in \mathcal{S}_k}  \mathbbm{1}\left(\left\langle\mathbf{w}_{y_{i}}, \mathbf{x}_{i}\right\rangle \leq \gamma+\max _{y^{\prime} \neq y_{i}}\left(\left\langle\mathbf{w}_{y^{\prime}}, \mathbf{x}_{i}\right\rangle+\epsilon\left\|\mathbf{w}_{y^{\prime}}-\mathbf{w}_{y_{i}}\right\|_{1}\right)\right) \geq 0.
    \label{non-neg}
\end{equation}
(\ref{non-neg}) implies that
\begin{align}  \label{thm2:2}
& \max_k \left[ \frac{1}{|\mathcal{S}_k|}\sum_{(x_i,y_i) \in \mathcal{S}_k}  \mathbbm{1}\left(\left\langle\mathbf{w}_{y_{i}}, \mathbf{x}_{i}\right\rangle \leq \gamma+\max _{y^{\prime} \neq y_{i}}\left(\left\langle\mathbf{w}_{y^{\prime}}, \mathbf{x}_{i}\right\rangle+\epsilon\left\|\mathbf{w}_{y^{\prime}}-\mathbf{w}_{y_{i}}\right\|_{1}\right)\right) \right] \notag \\
\leq &   \max_k\frac{1}{|\mathcal{S}_k|} \sum_{(x_i,y_i) \in \mathcal{S}}  \mathbbm{1}\left(\left\langle\mathbf{w}_{y_{i}}, \mathbf{x}_{i}\right\rangle \leq \gamma+\max _{y^{\prime} \neq y_{i}}\left(\left\langle\mathbf{w}_{y^{\prime}}, \mathbf{x}_{i}\right\rangle+\epsilon\left\|\mathbf{w}_{y^{\prime}}-\mathbf{w}_{y_{i}}\right\|_{1}\right)\right) \notag\\
= & \frac{K}{|\mathcal{S}|}\sum_{(x_i,y_i) \in \mathcal{S}}  \mathbbm{1}\left(\left\langle\mathbf{w}_{y_{i}}, \mathbf{x}_{i}\right\rangle \leq \gamma+\max _{y^{\prime} \neq y_{i}}\left(\left\langle\mathbf{w}_{y^{\prime}}, \mathbf{x}_{i}\right\rangle+\epsilon\left\|\mathbf{w}_{y^{\prime}}-\mathbf{w}_{y_{i}}\right\|_{1}\right)\right) .
\end{align}
Meanwhile, we can also yield
\begin{align} \label{thm2:3}
& \max_k \left[ \frac{2 W K}{\gamma}\left[\frac{\epsilon \sqrt{K} d^{\frac{1}{q}}}{\sqrt{|\mathcal{S}_k|}}+\frac{1}{|\mathcal{S}_k|} \sum_{y=1}^{K} \mathbb{E}_{\boldsymbol{\sigma}}\left[\left\| \sum_{(x_i,y_i) \in \mathcal{S}_k}  \sigma_{i} \mathbf{x}_{i} \mathbbm{1}\left(y_{i}=y\right)\right\|_{q}\right] \right]+3 \sqrt{\frac{\log \frac{2}{\delta}}{2 |\mathcal{S}_k|}}  \right] \notag \\
= & \frac{2 W K}{\gamma} \max_k \frac{1}{|\mathcal{S}_k|} \sum_{y=1}^{K} \mathbb{E}_{\boldsymbol{\sigma}}\left[\left\| \sum_{(x_i,y_i) \in \mathcal{S}_k}  \sigma_{i} \mathbf{x}_{i} \mathbbm{1}\left(y_{i}=y\right)\right\|_{q}\right] + \frac{2 W K^2 \epsilon d^{\frac{1}{q}}}{\gamma \sqrt{|\mathcal{S}|}}+ 3 \sqrt{\frac{K \log \frac{2}{\delta}}{2 |\mathcal{S}|}} \notag \\
\leq & \frac{2 W K}{\gamma} \max_{y,k} \frac{K}{|\mathcal{S}_k|} \mathbb{E}_{\boldsymbol{\sigma}}\left[\left\| \sum_{(x_i,y_i) \in \mathcal{S}_k}  \sigma_{i} \mathbf{x}_{i} \mathbbm{1}\left(y_{i}=y\right)\right\|_{q}\right] + \frac{2 W K^2 \epsilon d^{\frac{1}{q}}}{\gamma \sqrt{|\mathcal{S}|}}+  3 \sqrt{\frac{K \log \frac{2}{\delta}}{2 |\mathcal{S}|}} \notag \\
= & \frac{2 W K^3}{\gamma |\mathcal{S}|} \max_{y,k} \mathbb{E}_{\boldsymbol{\sigma}}\left[\left\| \sum_{(x_i,y_i) \in \mathcal{S}_k}  \sigma_{i} \mathbf{x}_{i} 1\left(y_{i}=y\right)\right\|_{q}\right] + \frac{2 W K^2 \epsilon d^{\frac{1}{q}}}{\gamma \sqrt{|\mathcal{S}|}}+  3 \sqrt{\frac{K \log \frac{2}{\delta}}{2 |\mathcal{S}|}}.
\end{align}
Combine (\ref{thm2:1})(\ref{thm2:2})(\ref{thm2:3}), we conclude this proof.
\end{proof}

\section{Experiments Settings}
\label{app:setting}
\subsection{Datasets and Networks}

$\textbf{CIFAR-10}.$ CIFAR-10 contains 60000 points of training data and 10000 of test data with 10 classes. There are 5000 training images and 1000 test images in each class. We split 300 images in each class from the training set as the validation set. We train ResNet-18 and WideResNet-34-10 for 100 epochs on CIFAR-10 and set the learning rate as 0.1.

$\textbf{CIFAR-100}.$ CIFAR-100 contains 60000 points of training data and 10000 of test data with 100 classes. There are 500 training images and 100 test images in each class. We split 30 images in each class from the training set to form the validation set. We train ResNet-18 for 100 epochs on CIFAR-100 and set the learning rate as 0.1.

\subsection{Hyper-parameters used in every method}
The model is trained under the perturbation radius $\epsilon_{train} =8/255$. The batch size is 128, perturbation step size is 0.007 and the number of iterations $K=10$. We use the SGD optimizer. The momentum is 0.9 and the weight decay is 2e-4. We evaluate the model by PGD-100 and CW attack. For PGD-100, we set perturbation radius $\epsilon_{test} = 8/255$ and step size is 0.003. For CW, we set perturbation radius $\epsilon_{test} = 8/255$ and step size is 0.003. For AutoAttack, we use the standard version of AA and set perturbation radius $\epsilon_{test} = 8/255$.

Following \cite{frl}, we set $\tau_1=\tau_2=0.05$, $\alpha_1=\alpha_2=0.05$ for $\textbf{FRL}$ on CIFAR-10. The best $\tau$ and $\alpha$ are chosen from $\{0.01, 0.03, 0.05, 0.07, 0.1\}$ for $\textbf{FRL}$ on CIFAR-100.
Moreover, following \citet{analysis:benz}, we set $\alpha=0.05$ for $\textbf{CSL}$ on CIFAR-10. The best $\alpha$ is chosen from $\{0.01, 0.03, 0.05, 0.07, 0.1\}$ for $\textbf{CSL}$ on CIFAR-100. For our method, the best $\eta$ is chosen from $\{1e-3, 5e-3, 1e-2, 5e-2, 1e-1, 5e-1\}$ for both CIFAR-10 and CIFAR-100.

\subsection{Hardware Specification and Environment}
Our experiments are conducted on a Ubuntu 64-Bit Linux workstation, having 10-core Intel Xeon Silver CPU (2.20 GHz) and 4 Nvidia GeForce RTX 2080 Ti GPUs with 11GB graphics memory.

\section{Supplementary Experiments}
\label{app:exp}

\begin{figure}[h!]
\centering
\subfigure[Ours V.S. TRADES]{
    \begin{minipage}[t]{0.23\linewidth}
    \centering
    \includegraphics[width=1.0\linewidth]{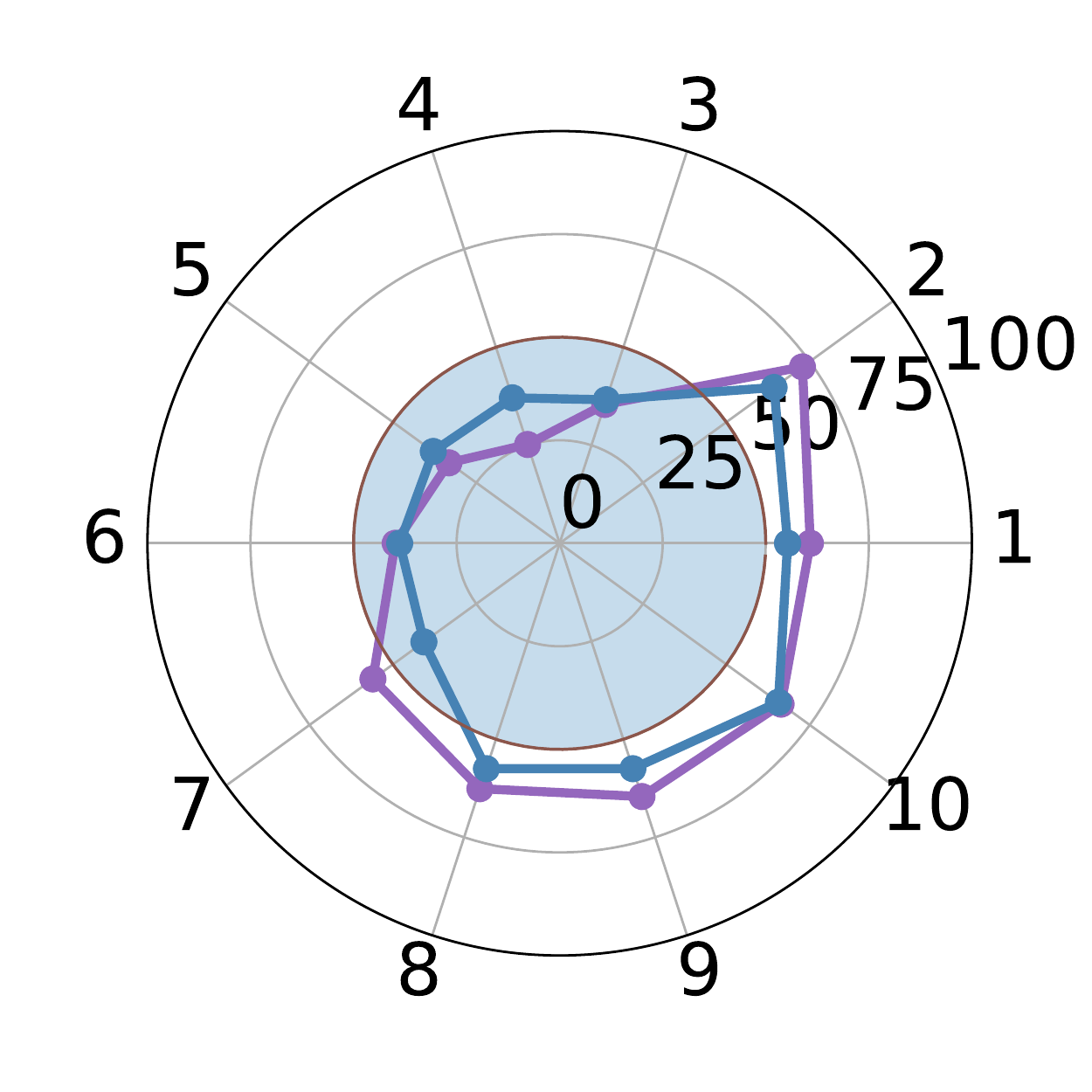}
    \label{fig:compare-pgd-cifar10-vs:trades}
    \end{minipage}
}
\subfigure[Ours V.S. CSL]{
    \begin{minipage}[t]{0.23\linewidth}
    \centering
    \includegraphics[width=1.0\linewidth]{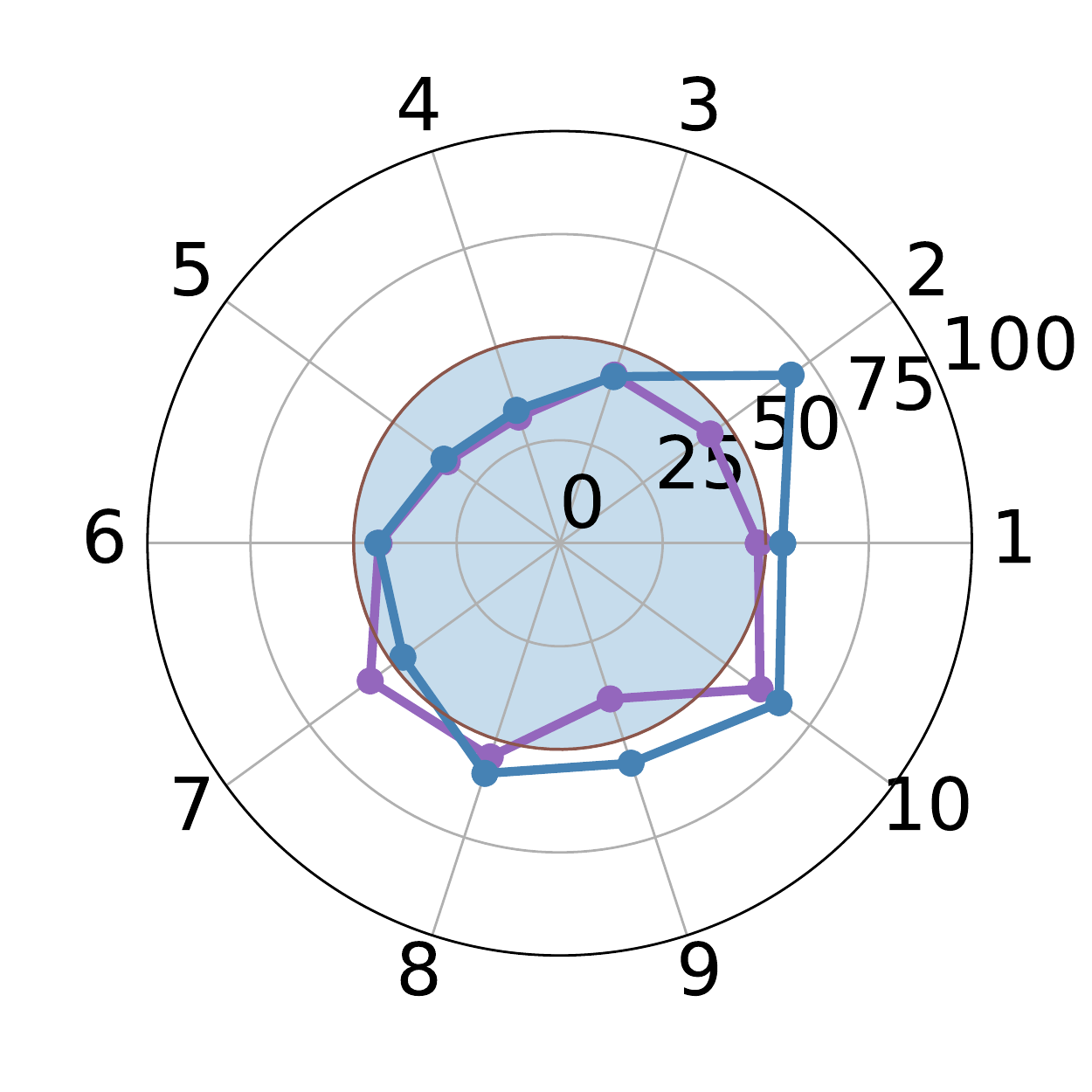}
    \label{fig:compare-pgd-cifar10-vs:csl}
    \end{minipage}
}
\subfigure[Ours V.S. FRL-RW]{
    \begin{minipage}[t]{0.23\linewidth}
    \centering
    \includegraphics[width=1.0\linewidth]{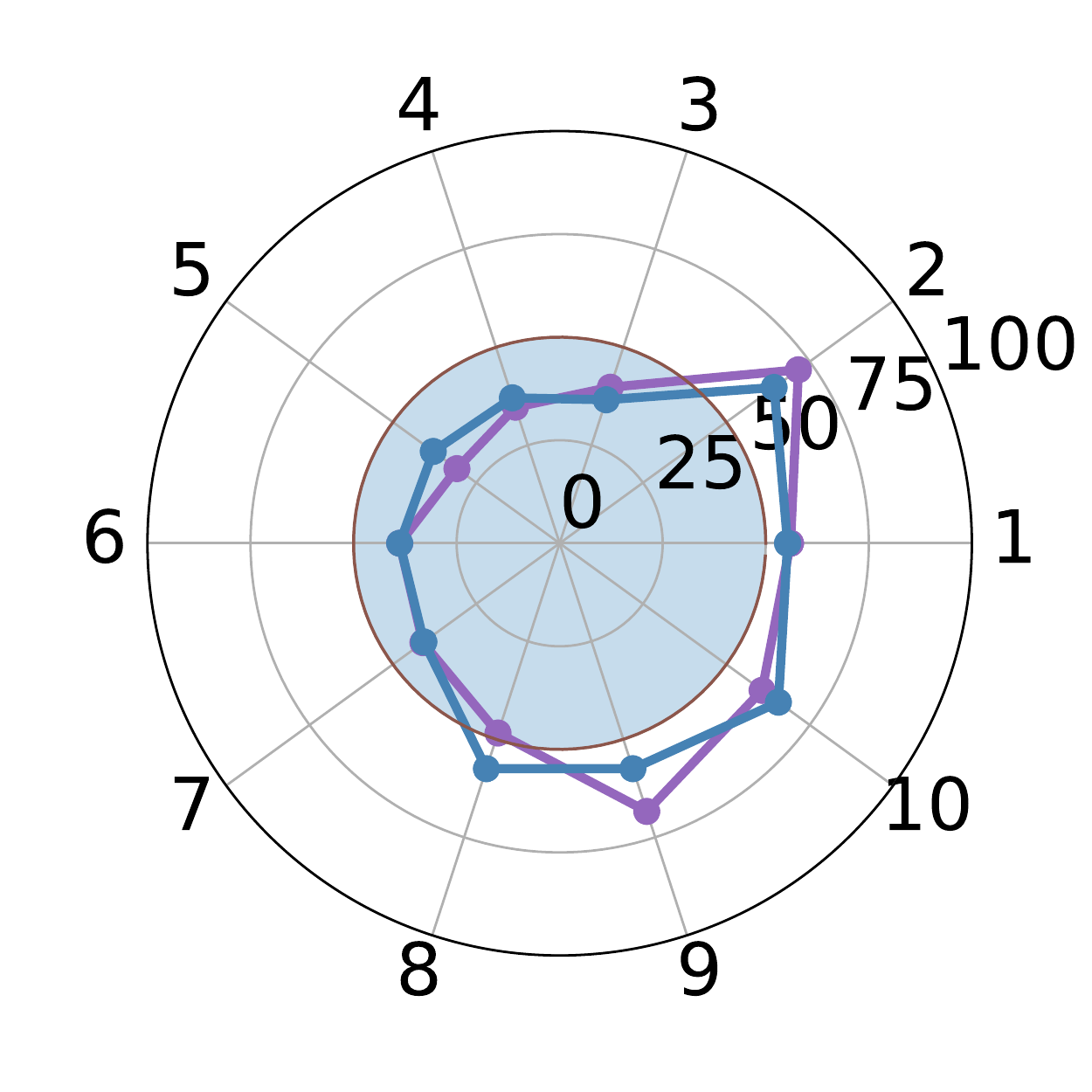}
    \label{fig:compare-pgd-cifar10-vs:rw}
    \end{minipage}
}
\subfigure[Ours V.S. FRL-RWRM]{
    \begin{minipage}[t]{0.23\linewidth}
    \centering
    \includegraphics[width=1.0\linewidth]{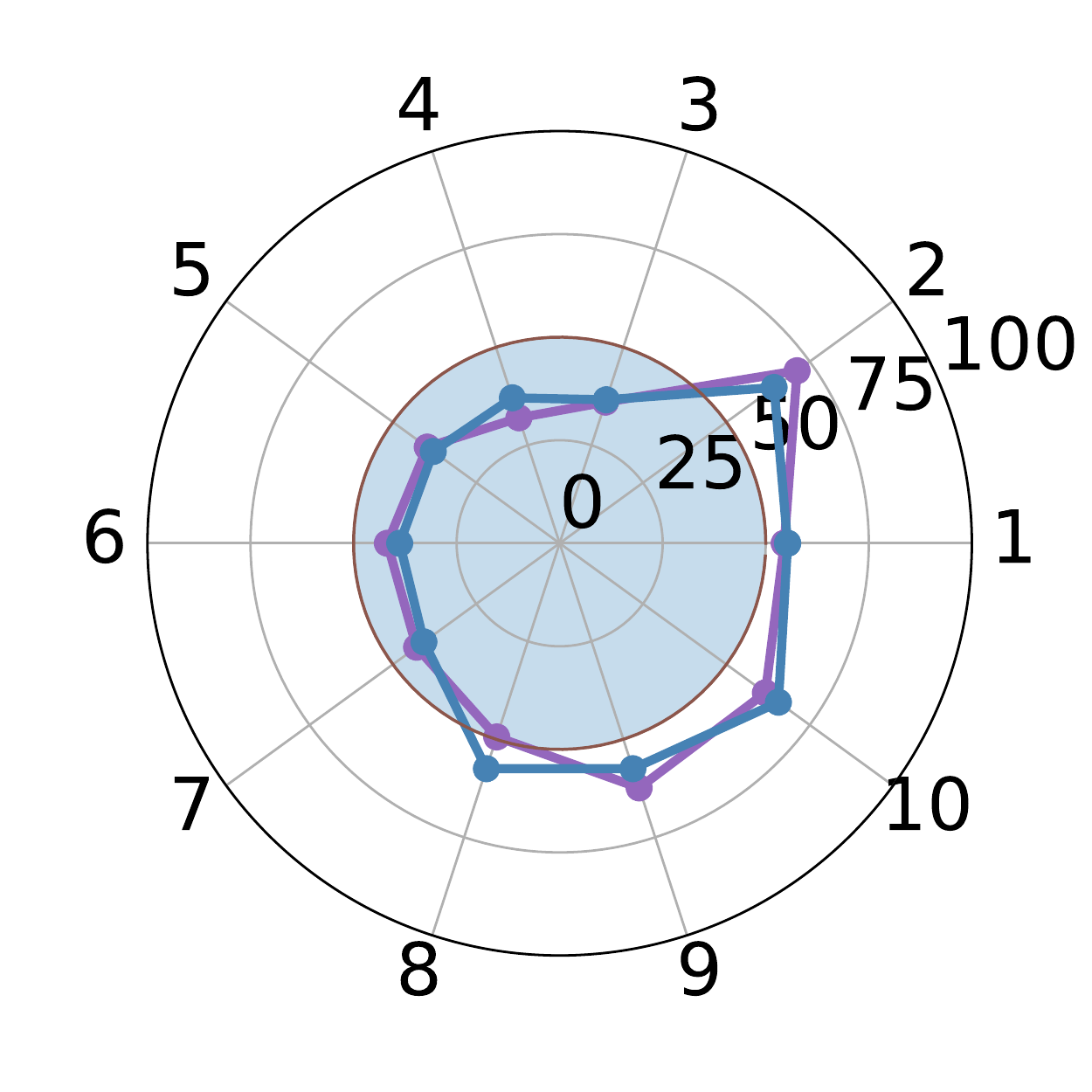}
    \label{fig:compare-pgd-cifar10-vs:rwrm}
    \end{minipage}
}
\caption{Class-wise robust accuracy disparity of all methods using ResNet-18 on CIFAR-10. We compare our method and another method in terms of the class-wise robust accuracy. We denote the results of our method with a blue line, while the results of the comparison methods are represented by a purple line. We evaluate robust accuracy under PGD-100 Attack.}
\label{fig:compare-pgd-cifar10}
\end{figure}

\begin{figure}[h!]
\centering
\subfigure[Ours V.S. TRADES]{
    \begin{minipage}[t]{0.23\linewidth}
    \centering
    \includegraphics[width=1.0\linewidth]{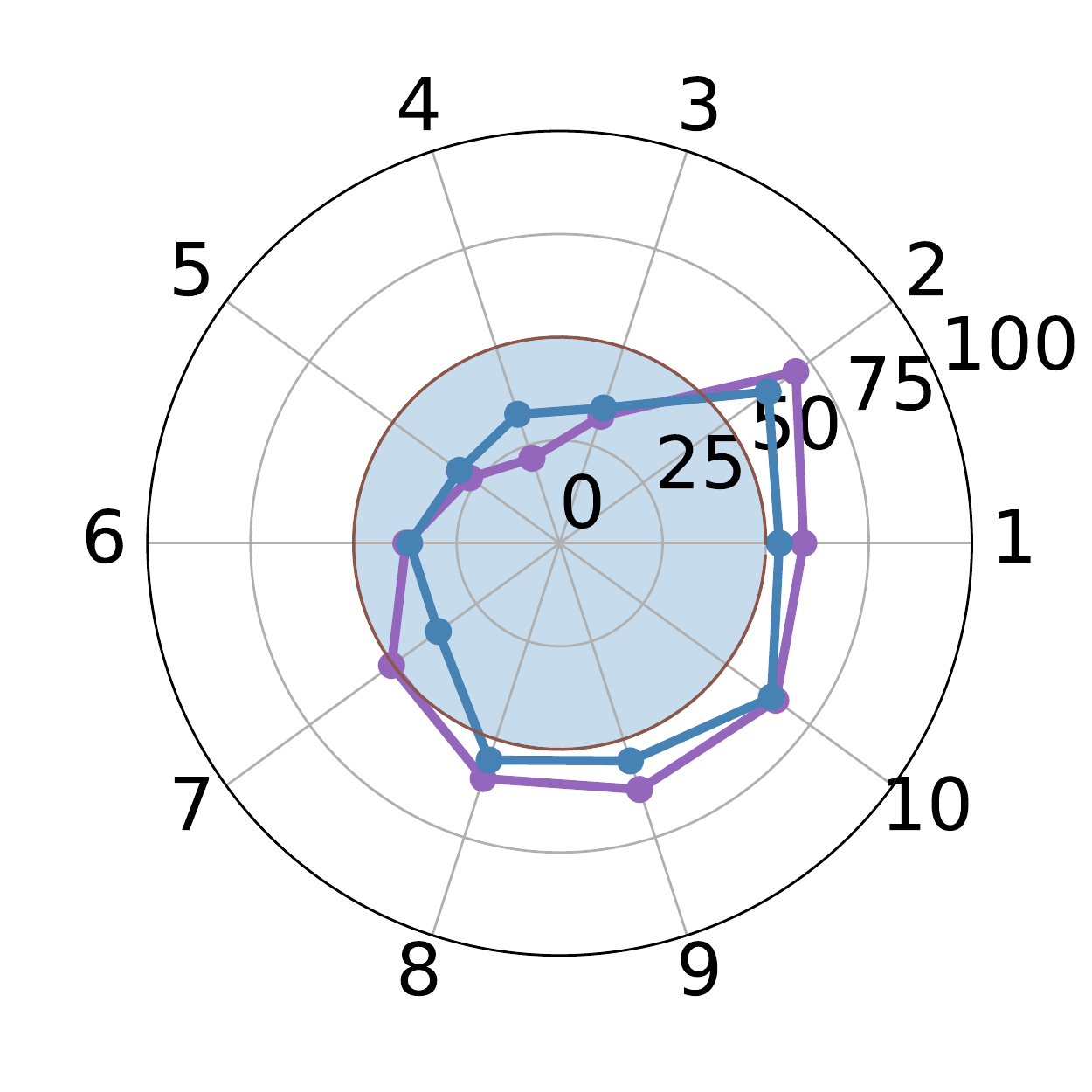}
    \label{fig:compare-AA-cifar10-vs:trades}
    \end{minipage}
}
\subfigure[Ours V.S. CSL]{
    \begin{minipage}[t]{0.23\linewidth}
    \centering
    \includegraphics[width=1.0\linewidth]{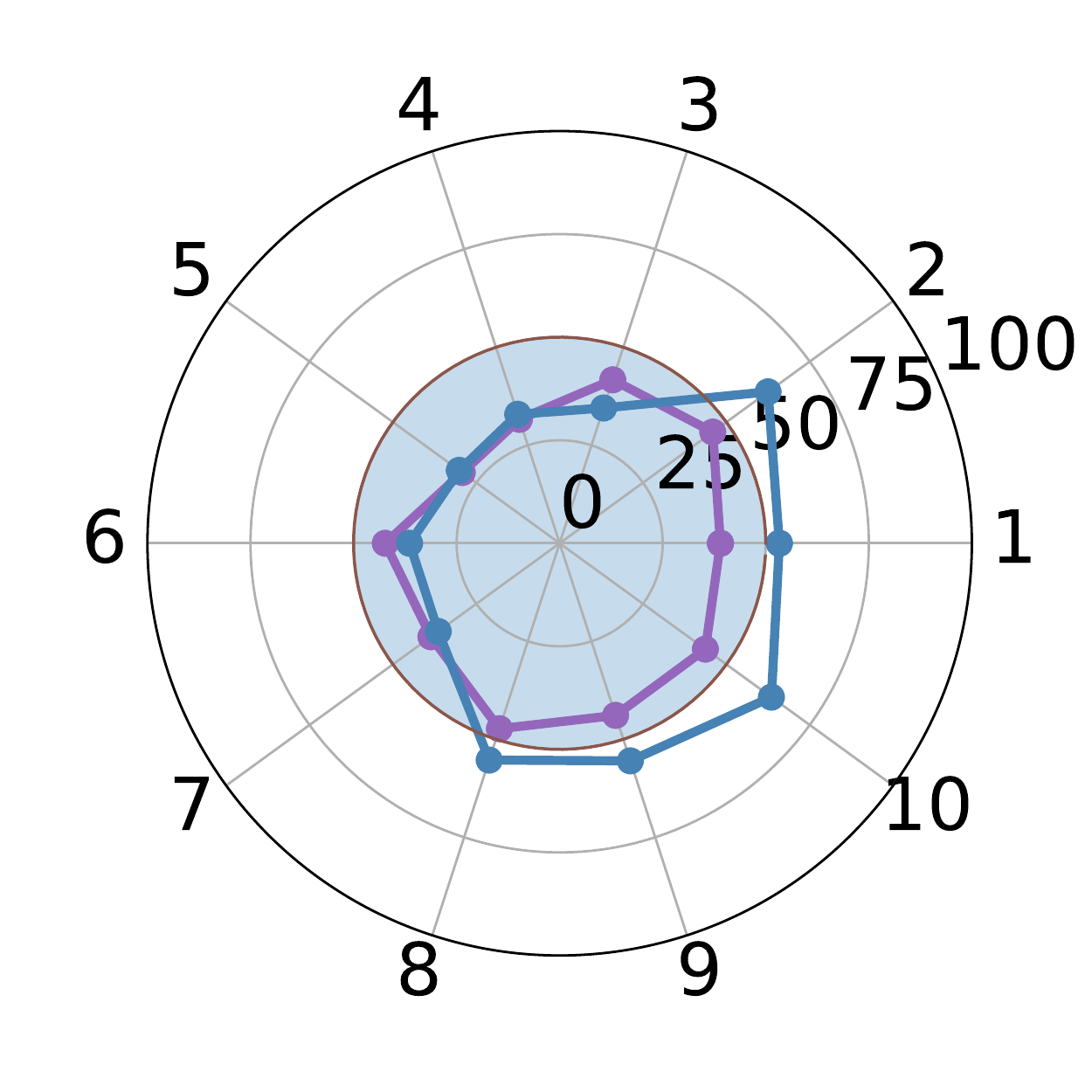}
    \label{fig:compare-AA-cifar10-vs:csl}
    \end{minipage}
}
\subfigure[Ours V.S. FRL-RW]{
    \begin{minipage}[t]{0.23\linewidth}
    \centering
    \includegraphics[width=1.0\linewidth]{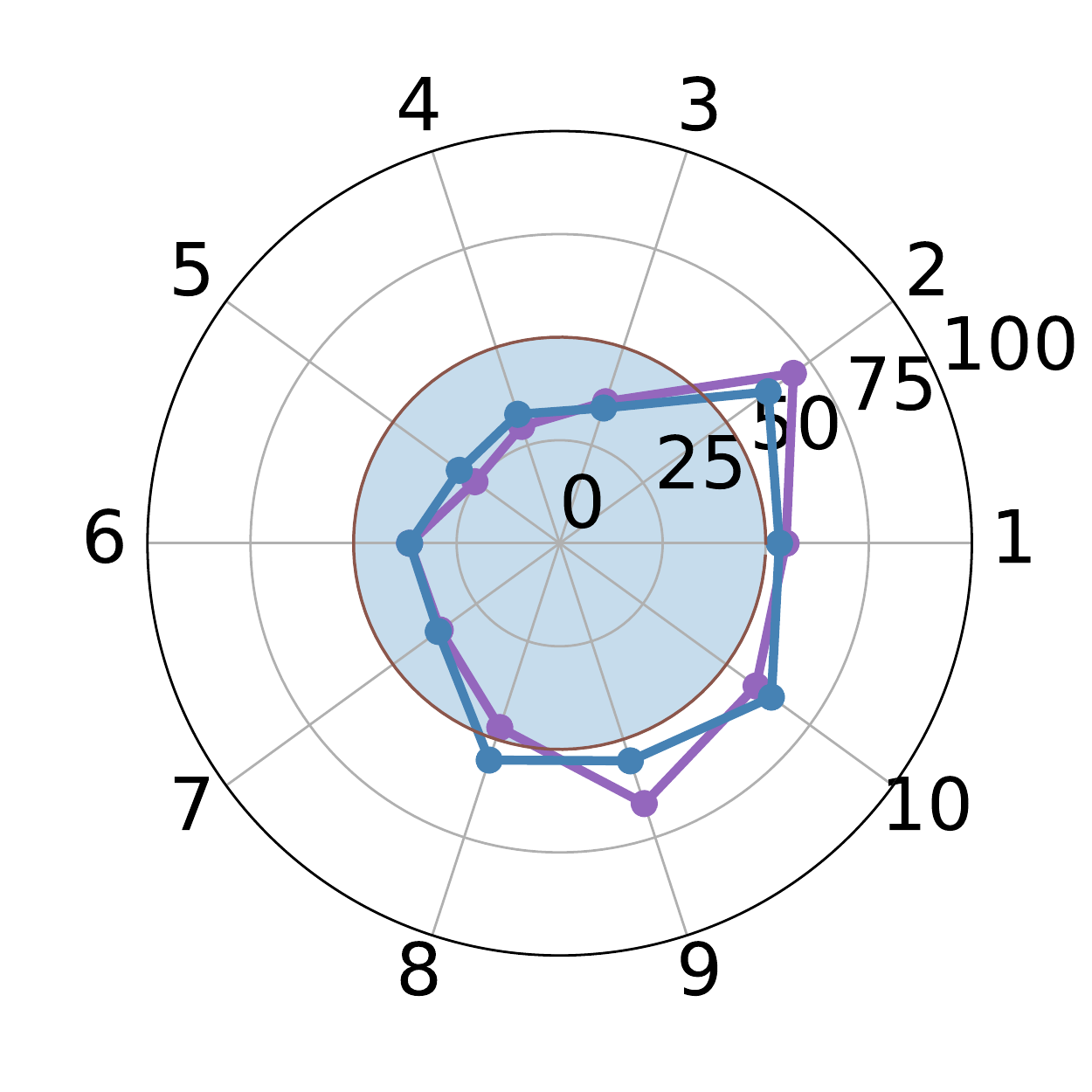}
    \label{fig:compare-AA-cifar10-vs:rw}
    \end{minipage}
}
\subfigure[Ours V.S. FRL-RWRM]{
    \begin{minipage}[t]{0.23\linewidth}
    \centering
    \includegraphics[width=1.0\linewidth]{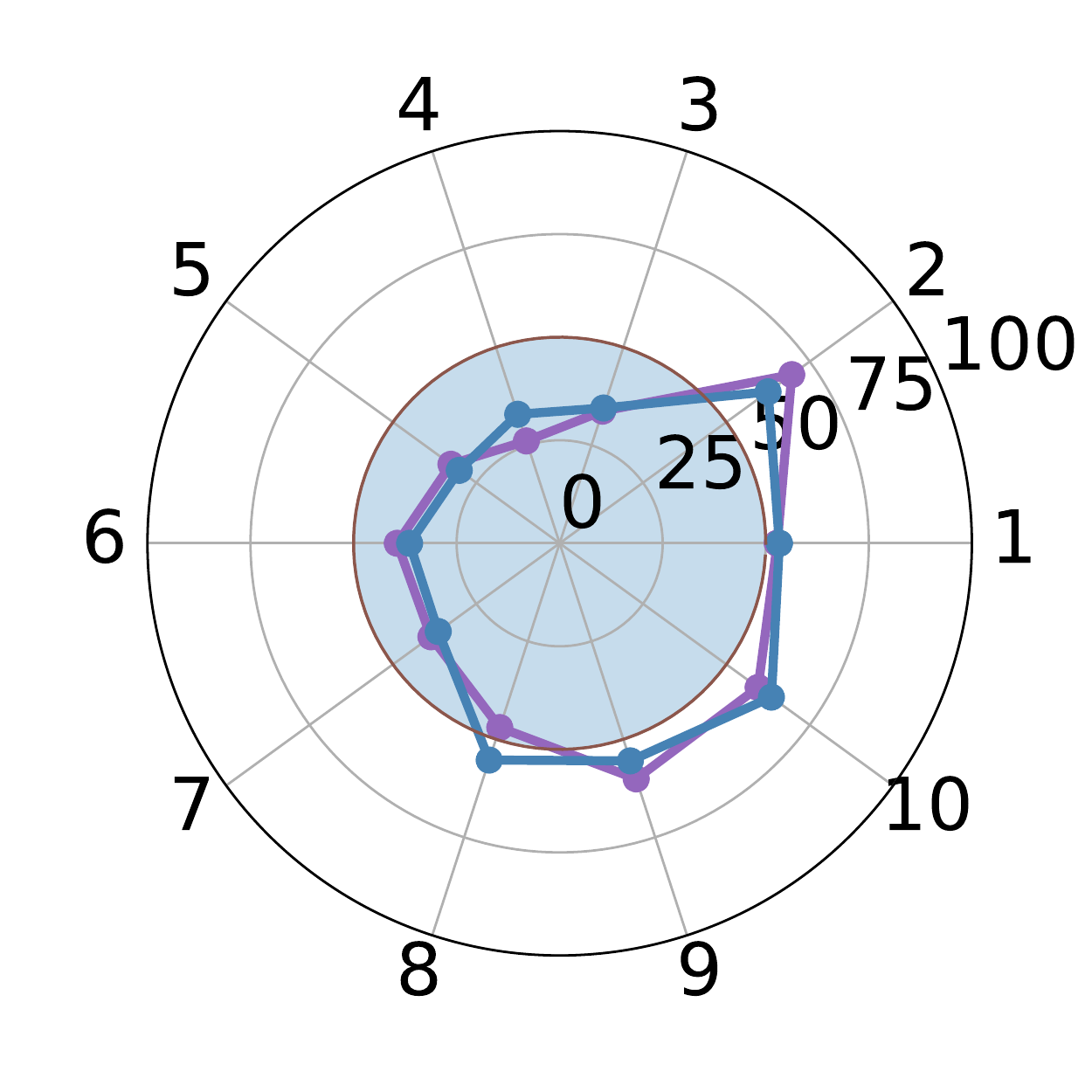}
    \label{fig:compare-AA-cifar10-vs:rwrm}
    \end{minipage}
}
\caption{Class-wise robust accuracy disparity of all methods using ResNet-18 on CIFAR-10. We compare our method and another method in terms of the class-wise robust accuracy. We denote the results of our method with a blue line, while the results of the comparison methods are represented by a purple line. We evaluate robust accuracy under AutoAttack.}
\label{fig:compare-AA-cifar10}
\end{figure}

In Figure \ref{fig:compare-pgd-cifar10}, we compare the class-wise robust accuracy evaluated by PGD-100 Attack between our method and all compared methods on CIFAR-10. As shown in Figure \ref{fig:compare-pgd-cifar10-vs:trades}, we find that our method achieves higher robust accuracy of class-4 and class-5 than TRADES, thus our method obtains a good performance on worst-class robust accuracy. In Figure \ref{fig:compare-pgd-cifar10-vs:csl}, \textbf{CSL} performs worse than our method in most of classes, which leads to a low average robust accuracy. From Figures \ref{fig:compare-pgd-cifar10-vs:rw} and \ref{fig:compare-pgd-cifar10-vs:rwrm}, we can see that our method achieves higher robust accuracy on class-4 than other two baselines, which is the most vulnerable class. Moreover, our proposed method outperforms other two baselines on class-8 and class-10 significantly, which contributes to the highest $\rho_{pgd}$ of our method.

In Figure \ref{fig:compare-AA-cifar10}, we compare the class-wise robust accuracy evaluated by AutoAttack between our method and all compared methods on CIFAR-10. As shown in Figure \ref{fig:compare-AA-cifar10-vs:trades}, we find that our method achieves higher robust accuracy of class-4 and class-5 than TRADES, thus our method obtains a good performance on worst-class robust accuracy. In Figure \ref{fig:compare-AA-cifar10-vs:csl}, \textbf{CSL} performs worse than our method in most of classes, which leads to a low average robust accuracy. From Figures \ref{fig:compare-AA-cifar10-vs:rw} and \ref{fig:compare-AA-cifar10-vs:rwrm}, we can see that our method achieves higher robust accuracy on class-4 than other two baselines, which is the most vulnerable class. Moreover, our proposed method outperforms other two baselines on class-8 and class-10 significantly, which contributes to the highest $\rho_{AA}$ of our method.

\subsection{More Results on $CV$ and $\rho$}
\label{app:cv}

\begin{table}[h]
\caption{Comparison results between $CV$ and $\rho$ using ResNet-18 on CIFAR-10.}
\vskip -0.25in
\begin{center}
\setlength{\tabcolsep}{1mm}{
\begin{tabular}{l|cccc|cccc|cccc}

\toprule
CIFAR-10 & \multicolumn{4}{c}{Natural} & \multicolumn{4}{c}{PGD-100 Attack}  & \multicolumn{4}{c}{AutoAttack} \\
\midrule
Method & Avg & Wst & $CV_{nat}$ & $\rho_{nat}$ & Avg. & Wst. & $CV_{pgd}$ & $\rho_{pgd}$ & Avg. & Wst. & $CV_{AA}$ & $\rho_{AA}$\\
\midrule
TRADES      & \textbf{82.11} & 64.6 & 0.0090  & 0 & \textbf{51.69} & 25.2 & 0.0250  & 0 & \textbf{48.64} & 21.7 &0.0278  & 0 \\
\midrule
FRL-RW      & 81.75 & 69.2 & 0.0148  & \textbf{0.067} & 49.02 & 30.8 & 0.0186  & 0.171 & 46.08 & 25.4 & 0.0222  & 0.118 \\
FRL-RWRM    & 80.69 & \textbf{71.4} & 0.0151  & 0.088 & 49.16 & 32.0 & 0.0150  & 0.221  & 45.94 & 26.1 & 0.0181  & 0.147\\
CSL         & 76.29 & 67.1 & \textbf{0.0018}  & -0.032 & 43.30 & 33.8 & \textbf{0.0024}  & 0.179 & 40.32 & 29.2 & \textbf{0.0031}  & 0.175 \\
Ours        & 80.98 & 69.5 & 0.0037  & 0.062 & 49.13 & \textbf{36.6} & 0.0129  & \textbf{0.403} & 46.04 & \textbf{30.1} & 0.0155  & \textbf{0.334} \\
\bottomrule
\end{tabular}
}
\label{tab:variance-cifar10}
\end{center}
\end{table}
From the results in Table \ref{tab:variance-cifar10}, we can see that \textbf{CSL} obtains the lowest $CV_{nat}$ value, while the average natural accuracy of \textbf{CSL} is the worst. $CV_{nat}$ is not a good measurement because it does not consider the trade-off between average natural accuracy and worst-class natural accuracy while $\rho_{nat}$ is a more reasonable measurement than $CV_{nat}$ by considering average natural accuracy and worst-class natural accuracy at the same time. Under both PGD-100 attack and AutoAttack, we find the similar result on $\rho$ and $CV$. $\rho_{pgd}$ is a more reasonable measurement than $CV_{pgd}$ and $\rho_{AA}$ is a more reasonable measurement than $CV_{AA}$ as well.

\end{document}